\documentclass{ieeeaccess_preprint}
\usepackage{cite}
\usepackage{amsmath, amssymb}
\usepackage{graphicx}
\usepackage{calc}
\usepackage[ruled, norelsize, vlined, nokwfunc]{algorithm2e}
\usepackage{subcaption}
\usepackage{amsthm}
\usepackage{bm}
\usepackage{xfp}
\usepackage{url}
\makeatletter
\gdef\@tfootnoteextra{This is an author-prepared preprint of the article published in \emph{IEEE Access}, vol.~11, pp.~128106--128124, 2023, doi: 10.1109/ACCESS.2023.3331685. IEEE Xplore: \url{https://ieeexplore.ieee.org/document/10314493}. 
The published article is licensed under CC BY-NC-ND 4.0.}
\newcommand{\mylabel}[2]{#2\def\@currentlabel{#2}\label{#1}}
\makeatother
\newtheorem{theorem}{Theorem}
\newtheorem{lemma}{Lemma}
\newtheorem{proposition}{Proposition}
\DeclareMathOperator*{\argmax}{argmax}
\DeclareMathOperator*{\argmin}{argmin}
\newlength{\IEEEspace}
\renewcommand{\thealgocf}{}
\begin{document}
\history{Received 21 October 2023, accepted 6 November 2023, date of publication 9 November 2023, date of current version 20 November 2023.}
\doi{10.1109/ACCESS.2023.3331685}
\title{Sequential Minimal Optimization Algorithm for One-Class Support Vector Machines With Privileged Information}
\author{Andrey Lange\authorrefmark{1,2},
        Dmitry Smolyakov\authorrefmark{3}, 
        and Evgeny Burnaev\authorrefmark{1,4}}
\address[1]{Skolkovo Institute of Science and Technology (Skoltech), Moscow 121205, Russia}
\address[2]{Federal Research Center ``Computer Science and Control" of Russian Academy of Sciences (FRC CSC RAS), Moscow 119333, Russia}
\address[3]{Institute for Information Transmission Problems of Russian Academy of Sciences (IITP RAS), 127994 Moscow, Russia}
\address[4]{Artificial Intelligence Research Institute (AIRI), Moscow 105064, Russia}
\tfootnote{This work was supported by the Analytical center under the RF Government (subsidy agreement 000000D730321P5Q0002, Grant No. 70-2021-00145 02.11.2021).}
\markboth
{Lange \headeretal: {Sequential Minimal Optimization} algorithm for~{one-class} {Support Vector Machines} with~privileged information}
{Lange \headeretal: {Sequential Minimal Optimization} algorithm for~{one-class} {Support Vector Machines} with~privileged information}
\corresp{Corresponding author: Andrey Lange (e-mail: a.lange@skoltech.ru).}
\begin{abstract}
One of the powerful techniques in data modeling is accounting for features that are available at the training stage, but are not available when the trained model is used to classify or predict test data --- \textit{Learning Using Privileged Information} paradigm (LUPI, Vapnik and Vashist~\cite{vapnik2009new}). 
\textit{Sequential Minimal Optimization} (SMO) method has been already developed for supervised Support Vector Machines (SVM) in Platt~\cite{platt1998fast} and Keerthi et al.~\cite{keerthi2001improvements}, for unsupervised (one-class) SVM in Sch{\"o}lkopf et al.~\cite{scholkopf2001estimating}, 
and for SVM with privileged information (SVM+) in Pechyony and Vapnik~\cite{pechyony2011fast}.
As can be seen, the missing brick in this research has long been a one-class SVM with privileged information (OC-SVM+).
In this paper, we propose SMO algorithm for OC-SVM+ that significantly outperforms non-sequential algorithms for training the OC-SVM+ model. 
Its finite-time convergence is established.
The experiments show how privileged information affects a descriptive domain in the space of original features.
Comparative benchmark tests demonstrate that our algorithm is superior over interior point algorithms.
\end{abstract}
\titlepgskip=-21pt
\maketitle
\section{Introduction} \label{sec_intro}
\subsection{Privileged information}
A new learning paradigm, called {\it Learning Using Privileged Information} (LUPI, 
{privileged information is sometimes called} {\it {hidden information}} or {\it{side information)}} was proposed by Vapnik~\cite{vapnik2006estimation},
Vapnik and Vashist~\cite{vapnik2009new}, and
Vapnik et al.~\cite{Vapnik2008Learning}. 
This approach allows us to take into account features that are not available at the test phase, but are available at the train phase.
This information can, however, improve the accuracy of the model (Vapnik and Izmailov~\cite{vapnik2017knowledge}), its robustness
(Li et al.~\cite{li2020robust}), and accelerate the convergence (when asymptotically fewer patterns for training are required, Vapnik and Vashist~\cite{vapnik2009new}). 
For example, a {machine learning} model is used to make decisions about a patient before surgery, and the data for training the model can be supplemented with some knowledge that is available only after or during the surgical process.
Likewise, time series forecasting can use data ``from the future'' that is available at the time of training.
LUPI considers that the original patterns (also called primary or decision)
$x_k\in\mathcal{X}$, $k=1, \ldots, l$, in a data set of size $l$ are augmented with additional (privileged) patterns $x_k^*\in\mathcal{X^*}$.
Thus, the model operates with tuples $(x_k, x_k^*)$ or triplets $(x_k, x_k^*, y_k)$ if the labels $y_k$ are known.
It can be said that the training phase involves an intelligent teacher who provides the student with additional explanations, metaphors, comparisons, etc. 

A natural approach to account for privileged information that is unavailable
for test examples is to consider the task as a missing data problem.
It means that we could first train a model that predicts missing (privileged) features $x_k^*$ by the original $x_k$ and then replace $x_k^*$ with their predictions in the decision model trained on both $x_k$ and $x_k^*$.
However, Vapnik and Izmailov~\cite{vapnik2015learning_with} developed another approach, which based the LUPI paradigm on a mechanism of {\it knowledge transfer} from the space of teacher's explanations to the space of student's decisions. 
The authors have illustrated it for well known SVM classifier 
(Boser et al.~\cite{boser1992training}, Cortes and Vapnik~\cite{cortes1995support}).
This extends SVM to the so-called SVM+ (Vapnik~\cite{vapnik2006estimation},
Vapnik and Vashist~\cite{vapnik2009new}, 
Serra-Toro et al.~\cite{serra2014exploring},
Lapin at al.~\cite{lapin2014learning}), 
i.e.~SVM with privileged information (for a full list of abbreviations, see Appendix~\ref{appendix0}). 
Along with main decision rule of the model $f(x)$, this mechanism is characterized by a correcting function $f^*(x^*)$.
The idea is to identify a small number of elements in the privileged space $\mathcal{X^*}$ that well approximate $f^*(x^*)$ and transfer these elements to the decision space $\mathcal{X}$.
\subsection{Support Vector Machines} 
Support Vector Machines are used for both classification and regression in a supervised formulation, and for 
\textit{domain description / anomaly detection} in an unsupervised formulation.
SVM is not a probabilistic model in the sense that it describes the domain on which the data examples are distributed without reconstructing the probability distribution itself.
There are no assumptions about this distribution; in particular, it may be multimodal.
The supervised SVM models the data region of each class, and the unsupervised SVM models the region in which the bulk of the data is concentrated. 
Examples of data outside the domain are considered outliers (anomalies).

SVMs are classic examples of models allowing for kernalization (along with, for example, kernel Linear Regression and kernel Principle Component Analysis).
Thanks to nonlinear kernels, the domain region can be curvilinear and multiply connected. 
Kernel methods, unlike, for example, decision trees and their ensembles, rely on similarities between observations provided by kernels and have a {smaller} number of hyperparameters.
However, the advantages of random trees include the ability to work with missing data.
At the same time, the advantage of kernel methods is also the ability to work with categorical features\footnote{{a library for working with categorical kernels: https://github.com/lange-am/Categorical-Hamming-Kernels}} (Lodhi et al.~\cite{lodhi2002text}, Kondor and Lafferty~\cite{kondor2002diffusion}, Couto~\cite{couto2005kernel}) and combine different kernels including kernels from different subsets of features. 
In the case of random trees, binary trees are used as a rule. 
A standard approach to processing categorical features using binary random trees is {\it one-hot encoding} (Potdar et al.~\cite{potdar2017comparative}, Rodr{\'\i}guez et al.~\cite{rodriguez2018beyond}). 
It transforms each categorical feature into a set of several binary (dummy) ones, the number of which is equal to the number of unique categories in the original feature (its cardinality).
With high cardinality, the number of dummy variables may be too large to store such a data set in memory and lead to model overfitting.
In kernel methods, on the contrary, the size of the kernel matrix depends only on the number of data samples and does not depend on the number of unique values in each feature. 
Although high cardinality leads to high complexity in calculating kernel matrix values, this is largely compensated for if they are calculated in parallel.

SVM models have simple analytical formulations in terms of linearly constrained quadratic optimization problem with respect to dual coefficients.
From the values of these coefficients, one can interpret the role of the corresponding data example in the decision rule and its closeness to the decision boundary of the data domain (and hence the degree of anomaly in the unsupervised case).
The number of dual coefficients in SVM equals to the number of training examples.
Although the decision rule $f(x)$ in fact depends only on some of these coefficients (support vectors), the number of which is usually not so large, the optimization procedure has to operate on all of them.
Classical algorithms for quadratic optimization, such as interior point optimizers, work in \textit{offline} mode. 
This means that they require storing the entire quadratic form matrix or the entire training sample in memory, which is often not possible with large data sets.
However, SVMs allow for the {\it online} methods,
in which individual working (\mbox{sub-}) sets of variables are processed sequentially.
This avoids problems with insufficient memory and significantly reduces computation time.
{Among these methods, Sequential Minimal Optimization (SMO) is the most effective for existing SVM models}.
\subsection{Paper contribution and structure}
In this paper we have incorporated the privileged features into the unsupervised $\nu$-SVM model.
Our main result is that we have developed the SMO algorithm for the one-class SVM model with privileged information (OC-SVM+), previously proposed in Burnaev and Smolyakov~\cite{burnaev2016one}.
Experiments show, that our method significantly outperforms offline interior point optimizer.

We also studied the properties of the coefficients in the model solution and their relationship with the anomalousness of the data samples and the coefficient $\nu$ (Proposition~\ref{proposition1}).
Moreover, we have succeeded in establishing the convergence of SMO iterations in a finite number of steps, similar to how it was done for other SVMs.
OC-SVM+ supports shrinking and caching.
We investigate and compare different variants of caching rules.

The algorithm is implemented in OCSVM-PLUS library\footnote{{https://github.com/lange-am/OCSVM\_PLUS}}, 
which is freely accessible and includes instructions for installation and use.
It should be noted that there are very few computational tools for anomaly detection when some features are privileged. 
Our tool allows to analyze arbitrary tabular data without missing values.

The structure of the paper is as follows.
Section~\ref{sec_relwork} describes the main work in incorporating the privileged information into different machine learning models, as well as SMO algorithms for different SVMs.
In Section~\ref{sec_oc-svm} we describe different formulations of {one-class} SVM.
{Also,} we show the primal and dual formulations of {one-class} SVM with privileged information and investigate the properties of different dual coefficients.
Section~\ref{sec_gsmo} describes SMO algorithm and explores the properties of the optimization steps used later in the convergence study. 
This section also proves that the algorithm stabilizes, i.e. converges to some limit vector of coefficients in at least an infinite number of steps. Finally, this section provides a proof of finite time convergence.
Section~\ref{sec_experiments} contains the experiments showing how the hyperparameters that controls privileged information affect the resulting domain. 
It also compares the quality of anomaly detection with and without privileged information, as well as a comparative benchmark tests.
Discussion and future work are in Section~\ref{sec_discussion}.
Some concluding remarks are given in Section~\ref{sec_conclusion}.
\section{Related work} \label{sec_relwork}
\subsection{Privileged information}
In the Afterword to the second edition of the book
Vapnik~\cite{vapnik2006estimation} a new inference technology called \textit{Learning Hidden Information} (LHI) was introduced.
SVM+ algorithm was also formulated as an instrument for this new advanced approach.
Finally, this paradigm was formulated in Vapnik and Vashist~\cite{vapnik2009new}.

LUPI as a general paradigm can be applied to many existing machine learning models beyond Support Vector Machines, including ensembles of random trees and neural networks. 
Moreover, in each such case, upgrading the model so that it can transfer knowledge from privileged features during training is a difficult task.
This poses the following challenges:
propose (or modify) the optimization problem, study its properties, develop an optimization method and study the relative issues such as convergence and complexity, and compare the performance with LUPI-free models.
In the following studies, an existing machine learning models were revised or practical applied models were developed by incorporating privileged information (PI): 
\begin{itemize}
\item {\it Supervised:}
AdaBoost with {PI} (Chen et al.~\cite{chen2012boosting},
Liu et al.~\cite{liu2022adaboost}), 
SVM+ for risk modelling (Ribeiro et al.~\cite{ribeiro2010financial, ribeiro2012enhanced}),
SVM+ and multi-task learning (Liang and Cherkassky~\cite{liang2007learning, liang2008connection}, 
Liang et al.~\cite{liang2009predictive}, 
Cai and Cherkassky~\cite{cai2014SVM},
Tang et al.~\cite{tang2019retaining}),
Regression Forests for facial feature detection with privileged head pose or gender (Yang and Patras~\cite{yang2013privileged}),
image classification using privileged attributes, bounding box annotations, and textual descripting tags (Sharmanska et al.~\cite{sharmanska2013learning, sharmanska2014learning}), 
Li et al.~\cite{li2014exploiting},
Wang and Ji~\cite{wang2015classifier},
Yan et al.~\cite{yan2016image},
Rodr{\'\i}guez et al.~\cite{rodriguez2023fine}),
{\it structured} SVM (SSVM) prediction algorithm for image object localization using PI (Feyereisl et al.~\cite{feyereisl2014object}),
unifying distillation and PI (Lopez-Paz et al.~\cite{lopez2015unifying}),
multi-instance learning for action and event recognition with privileged web data (Niu et al.~\cite{niu2016exploiting}),
knowledge transfer for neural networks (Vapnik and Izmailov~\cite{vapnik2017knowledge}),
image object detection using PI (Hoffman et al.~\cite{hoffman2016learning}),
domain adaptation (Sarafianos et al.~\cite{sarafianos2017adaptive}),
multiview privileged SVMs (Tang et al.~\cite{tang2018multiview}),
deep learning under PI (Lambert et al.~\cite{lambert2018deep}),
PI for structured output prediction (Zhang et al.~\cite{Zhang2019Structured}),
label enhancement with multi-label learning (Zhu et al.~\cite{zhu2020privileged}),
PI for the diagnosis of Alzheimer's disease (Li et al.~\cite{li2019learning}, Ganaie and Tanveer~\cite{ganaie2022ensemble}), breast (Shaikh et al.~\cite{shaikh2020transfer}) and liver (Zhang et al.~\cite{zhang2023multi}) cancers,
PI for image super-resolution using CNNs (Lee et al.~\cite{lee2020learning}),
robust SVM+ (Li et al.~\cite{li2020robust}, Wu et al.~\cite{wu2021lr}),
{\it twin} SVM with PI (Che et al.~\cite{che2021twin}),
robust twin SVM+ (Li et al.~\cite{li2021r}),
Support Vector Regression with variational PI (Shu et al.~\cite{shu2021v}),
PI for time series prediction (Jung and Johansson~\cite{jung2022efficient});
\item {\it Unsupervised:} 
clustering the images of handwritten digits using poetic descriptions as privileged data
(Feyereisl and Aickelin~\cite{feyereisl2012privileged}),
one-class SVM applied to grouped data and treating the group category as PI (Zhu and Zhong~\cite{zhu2014new}),
PI using metric learning (Fouad et al.~\cite{fouad2013incorporating}),
Marcacini et al.~\cite{marcacini2014privileged},
Yang et al.~\cite{yang2017person}),
incorporating PI to Isolation Forest (Shekhar and Akoglu~\cite{shekhar2018incorporating}),
one-class SVM with PI for malware detection (Burnaev and Smolyakov~\cite{burnaev2016one}), 
sensor fault detection in a road weather information system (Smolyakov et al.~\cite{smolyakov2018anomaly}),
and unsupervised domain adaptation in semantic segmentation using neural networks (Vu et al.~\cite{vu2019dada}),
and via distilled discriminative clustering (Tang et al.~\cite{tang2022unsupervised}),
kernel Ridge Regression based one-class classification using PI (Gautam et al.~\cite{gautam2019ockelm}).
\end{itemize}
\subsection{SMO in SVM models}
For the classical supervised SVM, 
the methods of sequential optimization were first proposed by Vapnik~\cite{vapnik2006estimation} (the chunking method), 
then by Osuna et al.~\cite{osuna1997improved} and Joachims~\cite{joachims1998making}, 
which algorithms maintain a fixed working set size.
Finally, Platt~\cite{platt1998sequential} proposed the Sequential Minimal Optimization (SMO), 
which works with only two coefficients at each iteration.
Further enhancements to online SVM have been obtained by Keerthi et al.~\cite{keerthi2001improvements}, 
who suggested an improved choice of the pair of coefficients violating the Karush-Kuhn-Tucker conditions, 
and then by Bordes et al.~\cite{bordes2005fast}.
However, in these studies, the optimal violating pair was in fact selected based on the first-order approximation of the quadratic objective.
Further studies of Fan et al.~\cite{fan2005working},
Palagi and Sciandrone~\cite{palagi2005convergence}, 
as well as Glasmachers et al.~\cite{glasmachers2006maximum} 
have relied on second-order approximations, and 
LIBSVM library (according to Chang and Lin~\cite{Chang2011LIBSVM}) is based on Fan's et al. work.

Along with supervised Support Vector Machines for classification and regression, there are unsupervised versions:
SVDD (Support Vector Domain Description) proposed by Tax and Duin~\cite{tax1999support, tax1999data} and further revisited by Chang et al.~\cite{chang2013revisit}, as well as  
$\nu$-SVM by Sch{\"o}lkopf et al.~\cite{scholkopf1999support}.
These models belong to one-class classification models (Khan and Madden~\cite{khan2014one}) used for anomaly detection (Aggarwal~\cite{aggarwal2017introduction}) and usually called OC-SVMs.
They describe the domain on which the data examples $x_k$ are distributed without reconstructing the probability distribution itself.
SMO algorithms for supervised SVMs of Platt~\cite{platt1998sequential} and Keerthi et al.~\cite{keerthi2001improvements} were extended to the case of unsupervised one-class SVM
in the studies of Sch{\"o}lkopf et al.~\cite{scholkopf2001estimating} and Keerthi and Gilbert~\cite{keerthi2002convergence} respectively.

The models for one-class SVM with privileged information, closest to our OC-SVM+, were considered in
Zhu and Zhong~\cite{zhu2014new}, and Zhang~\cite{zhang2015support}. 
Burnaev and Smolyakov~\cite{burnaev2016one} compared these algorithms with the offline version of OC-SVM+ and found that they are comparable in anomaly detection accuracy.
At the same time, these algorithms do not use a fast approach to the optimization procedure such as SMO.

Pechyony et al.~\cite{pechyony2010smo}, and Pechyony and Vapnik~\cite{pechyony2011fast} have described {\it generalized} SMO for supervised SVM with privileged information ({so-called} {gSMO}), 
which was previously used to perform experiments in Vapnik et al.~\cite{Vapnik2008Learning} and Vapnik and Vashist~\cite{vapnik2009new}.
Moreover, the authors have proposed new algorithms that use irreducible working sets selected with second order approximation and prove their convergence.

Finally, our SMO scheme for OC-SVM+ fills the gap in the above studies.
It finds the most violating pair from subsets of dual coefficients
(caches) that usually leave at-bound coefficients aside.
Similar to Pechyony et al.~\cite{pechyony2010smo},
we consider irreducible working sets of two original dual coefficients and two privileged ones.
This, however, allows us to generalize SMO convergence methodology of Keerthi and Gilbert~\cite{keerthi2002convergence} (with Takahashi and Nishi's~\cite{takahashi2005rigorous} fixes) 
and Lin~\cite{lin2001convergence, lin2002asymptotic} to LUPI model. 
\section{One-class SVM {models}} \label{sec_oc-svm}
\subsection{{One-class SVM formulations}}
Originally, SVM-based unsupervised model for outlier detection was proposed by 
Tax and Duin~\cite{tax1999support, tax1999data}. 
The data examples are modelled such that they are concentrated inside a sphere of the smallest radius $R$ and center $a$ in a feature space $\phi(\cdot)$, whereas a fraction $\nu$ of examples lie outside.
This leads to the following primal optimization problem:
 \begin{equation} 
    \label{svdd}
    \begin{IEEEeqnarraybox*}[\IEEEeqnarraystrutmode\IEEEeqnarraystrutsizeadd{\IEEEspace}{\IEEEspace}][c]{C}
        \min\limits_{\overline{R}, a, \xi_k} \> \overline{R} + \frac{1}{\nu l} \sum\limits_{k=1}^l\xi_k, \\
        \text{s.t.} \quad \|\phi(x_k) - a\|^2_{\ell_2} \leq \overline{R} + \xi_k, \quad \xi_k \geq 0.
    \end{IEEEeqnarraybox*}
\end{equation}
Here sphere-outside examples (i.e. that which violate $\|\phi(x_k)-a\|^2_{\ell_2} \leq \overline{R}$) are modelled by non-negative slack variables $\xi_k$ and the objective includes their overall penalty with the regularization multiplier $1/\nu l$.
$\overline{R}$ can be interpreted as the square of sphere radius $\overline{R}=R^2$.
According to Chang et al.~\cite{chang2013revisit}, for $\nu < 1$ the condition $\overline{R} \geq 0$ in \eqref{svdd} can be omitted since it is fulfilled automatically. 
Also, the objective in \eqref{svdd} is convex, unlike if the problem would be formulated with respect to $R$.

Using the inner product $K(x, y) = (\phi(x) \cdot \phi(y))$ the corresponding dual formulation is kernelized in the form of quadratic programming problem
\begin{IEEEeqnarray}{C}
        \min\limits_{\alpha} \> \frac{1}{\nu l}\sum\limits_{k=1}^l\sum\limits_{m=1}^l\alpha_k\alpha_m K(x_k, x_m)
        -\sum\limits_{k=1}^l \alpha_k K(x_k, x_k),\nonumber \\
        \label{svdd_dual}\\
        \text{s.t.} \quad \sum\limits_{k=1}^l\alpha_k = \nu l, \quad 0 \leq \alpha_k \leq 1.\nonumber 
\end{IEEEeqnarray}
Here $\alpha_k$ = 0 for sphere-insiders $x_k$, $\alpha_k$ = 1 for outsiders, $0 < \alpha_k < 1$ for on-sphere examples. 
The center $a$ and the decision function ($f(x)<0$ inside and $f(x)>0$ outside the sphere) are expressed through the dual variables $\alpha_k$:
\begin{IEEEeqnarray}{rCL}
    a &=& \frac{1}{\nu l}\sum\limits_{k=1}^l \alpha_k \phi(x_k), \nonumber \\ 
    f(x) &=& \|\phi(x) - a\|^2_{\ell_2}  - \overline{R} \nonumber \\
         &=& K(x, x) - \frac{2}{\nu l} \sum\limits_{k=1}^l\alpha_k K(x, x_k) + \|a\|^2_{\ell_2} - \overline{R}, \label{svdd_decision}
\end{IEEEeqnarray}
where $\|a\|^2_{\ell_2}=(1/\nu^2 l^2) \sum_{k=1}^l\sum_{m=1}^l\alpha_k \alpha_m K(x_k, x_m)$,
and $\overline{R}$ can be found using any on-sphere example $x_i$ ($0~<~\alpha_i <~1$) from the equation $f(x_i) = 0$. 
The decision function \eqref{svdd_decision} is calculated using only support vectors, i.e. $x_k$  such that $\alpha_k > 0$. 
From the constraints in \eqref{svdd_dual} it follows that $\nu \leq 1$, and this value is the upper bound of the fraction of anomalies $\alpha_k=1$,
and the lower bound of the fraction of support vectors $\alpha_k > 0$:
\begin{equation} \label{bounds_alpha}
    \sum_{\alpha_k=1}\frac{1}{l} \leq \sum_{0<\alpha_k<1}\frac{\alpha_k}{l} + \sum_{\alpha_k=1}\frac{\alpha_k}{l} = 
    \nu \leq \sum_{\alpha_k>0}\frac{1}{l}.
\end{equation}
Moreover, $\sum_{a_k=1}{1}/{l}\to\nu$ as $l\to\infty$, i.e. the fraction of anomalies tends to $\nu$ when the number of data examples is large.

The alternative statement ({so-called }$\nu$-{SVM}) was proposed by Sch{\"o}lkopf et al.~\cite{scholkopf1999support}.
The data examples are separated from the origin by a hyperplane with maximum margin in the kernel feature space:
\begin{equation} 
    \label{oneclass}
    \begin{IEEEeqnarraybox*}[\IEEEeqnarraystrutmode\IEEEeqnarraystrutsizeadd{\IEEEspace}{\IEEEspace}][c]{r?l}
        \min\limits_{w, \xi, \rho} & \frac{1}{2}\|w\|^2_{\ell_2} + \frac{1}{\nu l}\sum\limits_{k=1}^l\xi_k - \rho, \\
    \text{s.t.} & (w\cdot \phi(x_k)) \geq \rho - \xi_k, \quad\xi_k \geq 0.
    \end{IEEEeqnarraybox*}
\end{equation}
Again, anomalous data examples, i.e. on the origin side from the hyperplane $(w\cdot \phi(x_k))-\rho=0$ are penalized in the objective \eqref{oneclass} using slack variables $\xi_k$,
whereas ``normal'' examples are separated from anomalies with confidence $\rho$ aiming to be as big as possible.
The corresponding dual statement with kernel $K(x, y) = (\phi(x) \cdot \phi(y))$ has a form
\begin{equation}
    \label{oneclass_dual}
    \begin{IEEEeqnarraybox*}[][c]{r?l}
        \min\limits_{\alpha} & \sum\limits_{k=1}^l\sum\limits_{m=1}^l\alpha_k\alpha_m K(x_k, x_m) ,\\
        \text{s.t.} & \sum\limits_{k=1}^l\alpha_k = \nu l, \quad 0 \leq \alpha_k \leq 1,
    \end{IEEEeqnarraybox*}
\end{equation}
and the decision function is
\begin{equation} 
    \label{oneclass_decision}
    f(x) = (w\cdot \phi(x)) -\rho = \frac{1}{\nu l}\sum\limits_{k=1}^l\alpha_k K(x_k, x) - \rho.
\end{equation}
Similarly to above, the data examples $x_k$, for which $\alpha_k$ = 0, are ``normal'' examples $f(x_k)>0$, 
and $\alpha_k$ = 1 for anomalies $f(x_k)<0$, and $0 < \alpha_k < 1$ for on-hyperplane examples $f(x_k)=0$.
Optimal $\rho$ is found based on any on-hyperplane data example $x_i$: $\rho=1/(\nu l)\sum_{k=1}^l\alpha_k K(x_k, x_i)$. 
Only support vectors $\alpha_k>0$ contribute to \eqref{oneclass_decision}.
Since the constraints in \eqref{oneclass_dual} are the same as in \eqref{svdd_dual}, the parameter $\nu \leq 1$ has the same interpretation: the upper bound for 
the fraction of outliers and the lower bound for fraction of support vectors.
Moreover, for the kernels $K(x, y)$ such that $K(x, x)$ does not depend on $x$ the problems \eqref{svdd} and \eqref{oneclass} become equivalent since their dual formulations \eqref{svdd_dual} and \eqref{oneclass_dual} are the same.
{These kernels include kernels that depend only on the} $x-y$, {and, in particular, the radial basis function kernel} $K(x, y)=\exp(-\gamma\|x-y\|^2)$),
{which is one of the most commonly used kernels}.
\subsection{{One-class SVM with privileged information}}
By incorporating privileged information, OC-SVM \eqref{oneclass} is generalized to the following formulation (see Burnaev and Smolyakov~\cite{burnaev2016one}).
Similar to Vapnik's et al.~\cite{vapnik2015learning_with} SVM+, 
the slack variables are linearly modelled using privileged features $\xi_k= (w^*\cdot \phi^*(x_k^*)) + b^*$
and some known basis functions $\phi^*(x^*)$ defined by kernel $K^*(x^*, y^*) = (\phi^*(x^*)\cdot\phi^*(y^*))$.
The primal formulation is as follows:
\begin{equation}
    \label{primal_problem}
    \begin{IEEEeqnarraybox*}[\IEEEeqnarraystrutmode\IEEEeqnarraystrutsizeadd{\IEEEspace}{\IEEEspace}][c]{r?l}
    \min_{w, \rho, w^*, b^*, \zeta} & \frac{\nu l}{2} \|w\|^2 + \frac{\gamma}{2} \|w^*\|^2 - \nu l \rho\\
    \IEEEeqnarraymulticol{2}{r}{+ \sum_{k=1}^l \left( (w^*\cdot \phi^*(x_k^*)) + b^* + \zeta_k \right)},\\
    \text{s.t.} & (w \cdot \phi(x_k)) \geq \rho - (w^*\cdot \phi^*(x_k^*)) - b^*,\\
                & (w^* \cdot \phi^*(x_k^* ))+b^*+\zeta_k\geq 0,\\
                & \zeta_k \geq 0.
    \end{IEEEeqnarraybox*}
\end{equation}
Here the distances $\xi_k$ from the hyperplane $(w \cdot \phi(x_k))  = \rho$ in \eqref{oneclass} can be slightly negative $(w^* \cdot \phi^*(x_k^* )) + b^* \geq -\zeta_k$ and both the distance and the gap $\zeta_k \geq 0$ are penalized in the objective.
Similar to the original part, the regularization term $\|w^*\|^2$ is included with hyperparameter $\gamma>0$. 

The primal problem \eqref{primal_problem} is equivalent to the following dual one (see {Appendix} \ref{appendix1}):
\begin{equation} 
    \begin{IEEEeqnarraybox*}[\IEEEeqnarraystrutmode\IEEEeqnarraystrutsizeadd{\IEEEspace}{\IEEEspace}][c]{r?l}
  \min\limits_{\alpha, \delta} & \frac{1}{\nu l} \sum_{k, m}\alpha_k\alpha_m K(x_k, x_m) \\
   & \quad + \frac{1}{\gamma}\sum_{k, m}(\alpha_k - \delta_k)(\alpha_m - \delta_m) K^*(x_k^*, x_m^*),  \\
  \textrm{s.t.} &
    \sum_k \alpha_k = \sum_k \delta_k = \nu l, \\
    & \alpha_k \geq 0, \quad 0\leq \delta_k \leq 1. 
\end{IEEEeqnarraybox*}
\label{dual_problem}
\end{equation}
In the special case $\alpha_k=\delta_k$, OC-SVM+ optimization problem \eqref{dual_problem} becomes equivalent to the standard OC-SVM problem \eqref{oneclass_dual}.
The quadratic objective in \eqref{dual_problem} now uses two similarities between observations via kernels: 
between $x_k$ and $x_m$ in the original space and between $x_k^*$ and $x_m^*$ in the privileged space.
To solve problems \eqref{primal_problem}-\eqref{dual_problem} means to find the decision function $f(x)$, 
that is the same as \eqref{oneclass_decision} in OC-SVM, 
and the {\it correcting function}
\begin{IEEEeqnarray}{rCL}
    f^*(x^*) &=& {(w^*\cdot \phi^*(x^*))}{+b^*} \nonumber \\ 
    &=& \frac{1}{\gamma}\sum_{k=1}^l(\alpha_k - \delta_k)K^*(x^*_k, x^*) +b^*.
 \label{oneclass_correcting}
\end{IEEEeqnarray}

Zhu and Zhong~\cite{zhu2014new} proposed a modified version of one-class $\nu$-SVM \eqref{oneclass}--\eqref{oneclass_dual} which accounts for the group index of each data example by incorporating the group as privileged information. 
Their formulations are close to (\ref{primal_problem})--(\ref{dual_problem}),
however there are no privileged features $x^*$, and 
each group of data examples has its own set of the parameters $w^*$, $\phi^*$, and $b^*$. 
Similar approach proposed by Zhang~\cite{zhang2015support} incorporates privileged features based on SVDD formulation \eqref{svdd}--\eqref{svdd_dual}.

In the optimality of \eqref{dual_problem},
$\alpha_k$ and $\delta_k$ satisfy Karush-Kuhn-Tucker (KKT) conditions,
that can be written as follows (see {Appendix} \ref{appendix1}):
  \begin{align}
    f(x_k) + f^*(x_k^*) \geq 0, &\qquad \text{if} \quad  \alpha_k = 0; \label{alpha_0} \\
    f(x_k) + f^*(x_k^*) = 0, &\qquad \text{if} \quad  \alpha_k > 0; \label{alpha_g0} \\
    f^*(x_k^*) \leq 0, &\qquad \text{if} \quad  \delta_k=0;   \label{delta_0} \\
    f^*(x_k^*) = 0, &\qquad \text{if}    \quad  0<\delta_k<1; \label{delta_01} \\
    f^*(x_k^*) \geq 0, &\qquad \text{if} \quad  \delta_k=1.   \label{delta_1} 
\end{align}
This allows every coefficient to be checked independently 
using \eqref{alpha_0}--\eqref{alpha_g0} for $\alpha_k$ and \eqref{delta_0}--\eqref{delta_1} for $\delta_k$.
Also, $b^*$ can be found from \eqref{delta_01} using any data example $x_i^*$ with not at-bound coefficient $0<\delta_i<1$. 
Having $f^*(x^*)$ obtained, $\rho$ is found from \eqref{alpha_g0} using any example $x_j$ such that $\alpha_j>0$:
\begin{equation} 
    \label{b_rho}
    \begin{split}
    b^* &= -\frac{1}{\gamma}\sum_{k=1}^l(\alpha_k - \delta_k)K^*(x_k^*, x_i^*),\\ 
    \rho &= \frac{1}{\nu l}\sum\limits_{k=1}^l\alpha_k K(x_k, x_j) + f^*(x_j^*). 
    \end{split}
\end{equation}
To reduce an error we average over all $b^*_i$ and then over all $\rho_j$.
If there are no coefficients $0<\delta_k<1$, then we can set
\begin{equation} \label{b_star}
b^* = \frac{1}{2}\left( \min_{\delta_m=0} b^*_m + \max_{\delta_n=1} b^*_n \right).
\end{equation}
The existence of at least one coefficient $\delta_m=0$ and at least one coefficient $\delta_n=1$ when there are no coefficients $0<\delta_k<1$, as well as at least one of  $\alpha_j>0$ is guaranteed by the {equality constraints} in \eqref{dual_problem} and the fact that $0<\nu<1$. 

As it is required in LUPI, the domain is found in the space of original features.
The domain bound is defined as \mbox{$f(x)=0$}; \mbox{$f(x)>0$} for domain-insiders, $f(x)<0$ for anomalies.
\begin{proposition}\label{proposition1}
    Let $\alpha_k$, $\delta_k$, $k=1,\ldots,l$, are the solution of dual optimization problem \eqref{dual_problem}. Then
    \begin{IEEEenumerate}
        \item[\mylabel{prop1_d1}{(a)}] If $x_k$ is anomaly, then $\delta_k=1$; \item[\mylabel{prop1_d1_eq}{(b)}] If $\delta_k<1$, then $x_k$ is either domain-insider, or on-bound example;
        \item[\mylabel{prop1_3cond}{(c)}] Any of the conditions $\alpha_k>0$, $0<\delta_k<1$, and ``$x_k$ is on-bound'' follows from two others; 
        \item[\mylabel{prop1_nu_bound}{(d)}] $\nu$ is the lower bound of the fraction of coefficients \mbox{$\delta_k>0$} and the upper bound of the fraction of coefficients $\delta_k=1$;
        \item[\mylabel{prop1_nu_bound_anomaly}{(e)}] $\nu$ is the upper bound of the fraction of anomalies.
    \end{IEEEenumerate}
\end{proposition}
\begin{proof}
\ref{prop1_d1} follows from \eqref{alpha_0}, \eqref{alpha_g0} and \eqref{delta_1}, 
\ref{prop1_d1_eq} is the equivalence of \ref{prop1_d1},
\ref{prop1_3cond} follows from \eqref{alpha_g0} and \eqref{delta_01},
\ref{prop1_nu_bound_anomaly} follows from \ref{prop1_d1} and \ref{prop1_nu_bound}.
From $\sum_k \delta_k=\nu l$ and $0 \leq \delta_k \leq 1$ we have
\begin{gather*}
    \sum_{\delta_k=1}\frac{1}{l} \leq \sum_{0<\delta_k<1}\frac{\delta_k}{l} + \sum_{\delta_k=1}\frac{\delta_k}{l} = 
    \nu \leq \sum_{\delta_k>0}\frac{1}{l},
\end{gather*}
which gives \ref{prop1_nu_bound}.
\end{proof}

Note, that the converse to \ref{prop1_d1} is not true, and $\delta_k=1$ 
can either be inside, or be on-bound, or be an anomaly data example.
The property \ref{prop1_nu_bound_anomaly} is inherited from OC-SVM without privileged information \eqref{svdd_dual} and \eqref{oneclass_dual}. 
However, now $\nu$ is no longer the lower bound of the fraction of support vectors $\alpha_k>0$.
In OC-SVM+ they can also be inside the domain 
(see Fig.~\ref{fig:expr1_ocsvm_plus}, for example, the point with $\alpha_k>0$, $\delta_k=0$, 
which is domain-insider due to \ref{prop1_d1_eq} and \ref{prop1_3cond}).
By \ref{prop1_nu_bound_anomaly}, hereinafter we assume that $\nu > 1/l$.
\section{SMO algorithm {for OC-SVM+}} \label{sec_gsmo}
Generally, quadratic programming problem \eqref{dual_problem} can be solved by different traditional solvers.
However, there is an elegant solution that utilizes memory more efficiently.
In this paper, we present an adaptation of Platt's~\cite{platt1998sequential} Sequential Minimal Optimization (SMO) technique to OC-SVM with privileged information.
The main idea of SMO is to split the initial optimization problem into a sequence of smaller optimization problems.
By the way, an offline approach, unlike to online SMO, requires to operate with matrices of quadratic forms in \eqref{oneclass_dual} and \eqref{dual_problem} with the number of elements equal to the square of the number of training examples (whole batch size), which is often impossible.

In SMO every iteration considers a limited {\it working set} of dual coefficients.
Originally, for supervised SVM, Vapnik~\cite{vapnik2006estimation} suggested to utilize and chunk in memory only a subset of coefficients, 
since a significant part of them are usually zero.
The remaining coefficients are checked for optimality and those that violate are included in the working set.
But if the number of support vectors is too large, then the chunking method is not available. 
Osuna et al.~\cite{osuna1997improved} suggested to keep the size of the working set fixed by adding one example and excluding one example on every step.
And finally, Platt~\cite{platt1998sequential} suggested to work with the minimal working set, i.e. of size two.
In essence, this is the minimization of a multivariate quadratic function like (\ref{svdd_dual}), (\ref{oneclass_dual}), and (\ref{dual_problem}) using special {\it {coordinate-descent}} (Boyd and Vandenberghe~\cite{boyd2004convex}), in which two coordinates are involved at each iteration. 
In this case, the coordinates that most strongly violate the KKT conditions are selected.
Similarly, our SMO method uses working sets of pairs $\{\alpha_i, \alpha_j\}$ and $\{\delta_m, \delta_n\}$. 
In every pair the coefficients are shifted in two opposite directions by the same value aiming to optimize the objective and keeping the equality constraints in \eqref{dual_problem} satisfied.
\subsection{Optimization subproblems}
For a given pair of examples $i$ and $j$, 
one can find the optimal    
$t$ and $s$,
\begin{equation}
\begin{split} \label{gsmo_update}
\alpha_i \leftarrow \alpha_i + t, \quad
\alpha_j \leftarrow \alpha_j - t, \\
\delta_i \leftarrow \delta_i + s, \quad
\delta_j \leftarrow \delta_j - s.
\end{split}
\end{equation} 
This keeps the sums $\sum_k\alpha_k$=$\sum_k\delta_k$ in \eqref{dual_problem} unchanged.
Substituting the updates \eqref{gsmo_update} into the objective \eqref{dual_problem} 
leads to a local problem with respect to $t$ and $s$ (see the derivation in {Appendix} \eqref{appendix2}):
\begin{IEEEeqnarray}{r?ll} 
    \min\limits_{t, s} & \Phi(t, s) = t^2K_{ij} &+ 2t\Delta f_{ij} +(t-s)^2K_{ij}^* \nonumber \\
    && \qquad + 2(t - s)\Delta f_{ij}^*,    \vspace{0.2\baselineskip} \label{gsmo_objective} \\
  \text{s.t.} & 
  \IEEEeqnarraymulticol{2}{l}{
  \begin{IEEEeqnarraybox*}[\IEEEeqnarraystrutmode\IEEEeqnarraystrutsizeadd{\IEEEspace}{\IEEEspace}][c]{C} 
    -\alpha_i \leq t \leq \alpha_j, \label{gsmo_alpha_delta} \\
    -\delta_i \leq s \leq 1-\delta_i, \quad \delta_j-1 \leq s \leq \delta_j,
  \end{IEEEeqnarraybox*}} 
\end{IEEEeqnarray}
where we denote
\begin{gather}
   K_{ij}   = \frac{1}{\nu l}  \left( K(x_i, x_i) - 2K(x_i, x_j) + K(x_j, x_j) \right), \label{K}\\
   K_{ij}^* = \frac{1}{\gamma} \left( K^*(x_i^*, x_i^*) - 2K^*(x_i^*, x_j^*) + K^*(x_j^*, x_j^*) \right), \label{K_star} \\
   \Delta f_{ij} = f(x_i) - f(x_j), \quad  \Delta f^*_{ij} = f^*(x_i^*) - f^*(x_j^*), \label{Deltas} 
\end{gather}
and \eqref{gsmo_alpha_delta} follows from the constraints $\alpha \geq 0$ and $0 \leq \delta \leq 1$.

The objective \eqref{gsmo_objective} is a positive semi-definite quadratic function with respect to $s$ and $t$ because 
$K_{ij} = \|\phi(x_i) - \phi(x_j)\|^2_{\ell_2}/\nu l \geq 0$ and 
$K_{ij}^* = \|\phi^*(x_i^*) - \phi^*(x_j^*)\|^2_{\ell_2}/\gamma \geq 0$.
Therefore, it is convex and for $K_{ij}>0$, $K_{ij}^*>0$, a zero-gradient point provides a global minimum. 
Differentiating \eqref{gsmo_objective}, we have
\begin{align*}
\frac{\partial \Phi}{\partial t} &= 2tK_{ij} + 2\Delta f_{ij} + 2tK_{ij}^* - 2sK_{ij}^* + 2\Delta f^*_{ij} = 0,\\
\frac{\partial \Phi}{\partial s} &= 2sK_{ij}^*-2tK_{ij}^* - 2\Delta {f}_{ij}^* = 0.
\end{align*}
From these linear equations we obtain
\begin{equation}\label{unbound_gsmo}
t = -\frac{\Delta  {f}_{ij}}{K_{ij}}, \qquad
s = \frac{\Delta  {f}_{ij}^*}{K_{ij}^*} - \frac{\Delta {f}_{ij}}{K_{ij}}.
\end{equation}
If the optimal point \eqref{unbound_gsmo} is not feasible inside the box \eqref{gsmo_alpha_delta}, 
then it lies on the rectangle sides 
$t = -\alpha_i$, $t=\alpha_j$, 
$s=\max(-\delta_i, \delta_j-1)$, 
$s=\min(1-\delta_i, \delta_j)$.

However, optimizing original and privileged coefficients independently leads to an algorithm that uses only one-dimensional descents.
The update $\alpha_i \leftarrow \alpha_i + t$, $\alpha_j \leftarrow \alpha_j - t$ leads to
\begin{equation}
\begin{IEEEeqnarraybox*}[\IEEEeqnarraystrutmode\IEEEeqnarraystrutsizeadd{\IEEEspace}{\IEEEspace}][c]{r?ll}
  \min\limits_{t} & \varphi(t) = \Phi(t, 0) &= t^2(K_{ij}+K_{ij}^*) \\
    && \qquad + 2t(\Delta f_{ij}+\Delta f_{ij}^*),  \\
  \text{s.t.} & -\alpha_i \leq t \leq \alpha_j, &
\end{IEEEeqnarraybox*} 
\label{gsmo_objective_t}
\end{equation}
and the update $\delta_i \leftarrow \delta_i + s$, $\delta_j \leftarrow \delta_j - s$ corresponds to
\begin{equation}
  \begin{IEEEeqnarraybox*}[\IEEEeqnarraystrutmode\IEEEeqnarraystrutsizeadd{\IEEEspace}{\IEEEspace}][c]{r?l}
  \min\limits_{s} & \varphi^*(s) = \Phi(0, s) = s^2K_{ij}^* - 2s\Delta f_{ij}^*,  \\
  \text{s.t.} & \max(-\delta_i, \delta_j-1) \leq s \leq \min(1-\delta_i, \delta_j).
  \end{IEEEeqnarraybox*}
  \label{gsmo_objective_s}
\end{equation}

In all three cases there is no need to calculate and compare the values of the objective.
In the single-variable cases \eqref{gsmo_objective_t} and \eqref{gsmo_objective_s}
the vertex of the parabola is compared with the bounds.
In the case of two-variables \eqref{gsmo_objective}--\eqref{gsmo_alpha_delta}
one need to find which of 9 regions (the rectangle \eqref{gsmo_alpha_delta} and 8 adjacent areas) the unconstrained minimum \eqref{unbound_gsmo} belongs to.

Single-variable updates are more in tune with Platt's idea of irreducible working set, since only two coefficients $\alpha_i$, $\alpha_j$ or $\delta_i$, $\delta_j$ are involved in the optimization step.
These reduced sets make it easy to establish the convergence of the method.
\subsection{violating pairs}
Basically, if for a current approximation of some $\alpha_k$ or $\delta_k$ 
the values of $f(x_k)$ or $f^*(x_k^*)$ violate \eqref{alpha_0}--\eqref{delta_1}, 
then these coefficients need to be updated. 
For this we can write down the following requirements:
\begin{equation} 
\begin{IEEEeqnarraybox*}[\IEEEeqnarraystrutmode\IEEEeqnarraystrutsizeadd{\IEEEspace}{\IEEEspace}][c]{C}
    \begin{IEEEeqnarraybox*}[\IEEEeqnarraystrutmode\IEEEeqnarraystrutsizeadd{\IEEEspace}{\IEEEspace}][c]{rCl}
        \min_{\alpha_k > 0} (f(x_k) + f^*(x_k^*)) &=& 
        \max_{\alpha_k > 0} (f(x_k) + f^*(x_k^*)) \\ 
        & \leq & \min_{\alpha_k = 0} (f(x_k) + f^*(x_k^*));
    \end{IEEEeqnarraybox*}  \\
    \begin{IEEEeqnarraybox*}[\IEEEeqnarraystrutmode\IEEEeqnarraystrutsizeadd{\IEEEspace}{\IEEEspace}][c]{rCl}
        \max_{\delta_k = 0} f^*(x_k^*) \leq \min_{0<\delta_k<1} f^*(x_k^*) 
        &=& \max_{0 < \delta_k < 1} f^*(x_k^*)\\ &\leq& \min_{\delta_k = 1} f^*(x_k^*).
    \end{IEEEeqnarraybox*}
    \label{KKT_enequalities}
\end{IEEEeqnarraybox*}
\end{equation}

In practice, in SMO methods for SVM the requirements of this kind are weakened by introducing small tolerance parameter $\tau>0$ (for example, $\tau=10^{-3}$).
Using this, we can require that the maximal difference between the values of the functions in \eqref{alpha_g0} and \eqref{delta_01} that tend to zero in optimality (i.e. equalities in \eqref{KKT_enequalities}) be less than $\tau$:
\begin{gather*}
    \max_{\alpha_k > 0} (f(x_k) + f^*(x_k^*))  - \min_{\alpha_k > 0} (f(x_k) + f^*(x_k^*)) \leq \tau; \\
    \max_{0 < \delta_k < 1} f^*(x_k^*) - \min_{0 < \delta_k < 1} f^*(x_k^*) \leq \tau.
\end{gather*}
The same is about the actual exceedances of the left parts over the right ones in the inequalities \eqref{KKT_enequalities}:
\begin{equation*}
\begin{IEEEeqnarraybox*}[\IEEEeqnarraystrutmode\IEEEeqnarraystrutsizeadd{\IEEEspace}{\IEEEspace}][c]{rCl}
    \max_{\alpha_k > 0} (f(x_k) + f^*(x_k^*)) &-& \min_{\alpha_k = 0} (f(x_k) + f^*(x_k^*)) \leq \tau; \\
    \max_{\delta_k = 0} f^*(x_k^*) &-& \min_{0<\delta_k<1} f^*(x_k^*) \leq \tau, \\
    \max_{0<\delta_k<1} f^*(x_k^*) &-& \min_{\delta_k=1} f^*(x_k^*) \leq \tau, \\
    \max_{\delta_k=0} f^*(x_k^*)   &-& \min_{\delta_k=1} f^*(x_k^*) \leq \tau.    
\end{IEEEeqnarraybox*}
\end{equation*}
Summarizing over this, we obtain the equivalent criteria based on pairs of intersected sets of coefficients $\alpha_k > 0$ and $\alpha_k \geq 0$, 
as well as $\delta_k >0$ and $\delta_k < 1$:
\begin{gather*}
    \max_{\alpha_k > 0} (f(x_k) + f^*(x_k^*))  - \min_{\alpha_k \geq 0} (f(x_k) + f^*(x_k^*)) \leq \tau; \\
    \max_{\delta_k < 1} f^*(x_k^*) - \min_{\delta_k > 0} f^*(x_k^*) \leq \tau.
\end{gather*}
Thus, using $\Delta f_{ij}$, $\Delta f_{ij}^*$, we obtain
\begin{equation}
    \begin{IEEEeqnarraybox*}[\IEEEeqnarraystrutmode\IEEEeqnarraystrutsizeadd{\IEEEspace}{\IEEEspace}][c]{rCl}
    \IEEEeqnarraymulticol{3}{c}{\Delta f_{ij} + \Delta f_{ij}^* \leq \tau,} \\
    i&=&\argmax_{k|\alpha_k > 0} (f(x_k) + f^*(x_k^*)), \\
    j&=&\argmin_{k|\alpha_k \geq 0} (f(x_k) + f^*(x_k^*)), 
    \end{IEEEeqnarraybox*}
    \label{KKT_a} 
\end{equation}
and
\begin{equation}
    \begin{IEEEeqnarraybox*}[\IEEEeqnarraystrutmode\IEEEeqnarraystrutsizeadd{\IEEEspace}{\IEEEspace}][c]{C}
    \Delta f_{ij}^* \leq \tau, \\
    i=\argmax_{k|\delta_k < 1} f^*(x_k^*), \quad
    j=\argmin_{k|\delta_k > 0} f^*(x_k^*). 
    \end{IEEEeqnarraybox*}    
    \label{KKT_d} 
\end{equation}

If for some pair $i\in \{k~|~\alpha_k>0\}$, $j\in \{k~|~\alpha_k \geq 0\}$ we have a violation $\Delta f_{ij} + \Delta f_{ij}^* > \tau$, 
then $\alpha_i$, $\alpha_j$ (or simply $i$, $j$) is a {\it $\tau$-violating} pair for the original \mbox{($\alpha$-)KKT} conditions \eqref{alpha_0}--\eqref{alpha_g0}.
Similarly, if for some $i\in \{k~|~ \delta_k < 1\}$, $j\in \{k~|~ \delta_j > 0\}$ we have a violation $\Delta f_{ij}^* > \tau$, 
then $\delta_i$, $\delta_j$ ($i$, $j$) is a $\tau$-violating pair for the privileged \mbox{($\delta$-)KKT} \eqref{delta_0}--\eqref{delta_1}.
We call the pair $i$, $j$ from \eqref{KKT_a} an $\alpha$-pair, and the pair from \eqref{KKT_d} a $\delta$-pair.
If an $\alpha$-pair or a $\delta$-pair is $\tau$-violating, 
then its violation, respectively 
$\Delta f_{ij} + \Delta f_{ij}^* > \tau$ or 
$\Delta f_{ij}^* > \tau$, is the most among all pairs of data examples.
So, in OC-SVM with privileged information each of two violated conditions provides the corresponding most (worst) $\tau$-violating pair.
If both \eqref{KKT_a} and \eqref{KKT_d} are satisfied, then no $\tau$-violating pairs remained and the iterations stop. 
The obtained $\alpha_k$, $\delta_k$, $k=1,\ldots,l$, is a {\it $\tau$-approximate} solution.

In earlier works on Sequential Minimal Optimization for classical supervised Support Vector Machines (Platt~\cite{platt1998sequential,platt1998fast}) and for one-class Support Vector Machines (Sch{\"o}lkopf et al.~\cite{scholkopf2001estimating}), 
a pair of elements for local optimization was found not as the worst pair among all pairs of training examples.
These SMO algorithms iterate over all examples until a single data example violating KKT conditions (similar to \eqref{alpha_0}--\eqref{delta_1}) is found. 
Once it's found, a second one is chosen in the inner loop over only not at-bound examples ($\mathit{NB} = \{k~|~ 0<\alpha_k<C\}$ in classic supervised SVM).
For the first pattern the not at-bound set is also preferred, and if a violation is not found there, 
then the sets of bound coefficients $\alpha_k=0$ and $\alpha_k=C$ are included in the consideration.
These heuristics allow these algorithms to spend most of its time adjusting not at-bound coefficients,
which number is usually not so large.
Moreover, in these studies the tolerance $\tau$ is used to check only single-element KKT violations,
which also require to recalculate the intercepts such as $\rho$.
Further, Keerthi et al.~\cite{keerthi2001improvements} modified SMO for SVM by changing single-point KKT check to the search for the most $\tau$-violating pair of elements.
This resulted in remarkable speed improvement on several benchmark tests.

Similarly, our SMO for OC-SVM+ looks for the most $\tau$-violating pair,
although using two criteria: the original $\eqref{KKT_a}$ and the privileged $\eqref{KKT_d}$, unlike to only one.
Also, the descents \eqref{gsmo_objective}--\eqref{gsmo_alpha_delta} and \eqref{gsmo_objective_t}, \eqref{gsmo_objective_s}, 
as well as the search for the most violating pair do not require $\rho$ and $b^*$,
because the differences $\Delta f_{ij}$, $\Delta f_{ij}^*$ in \eqref{Deltas} and 
$\argmin$, $\argmax$ in $\eqref{KKT_a}$--$\eqref{KKT_d}$ do not actually depend on them.
The extreme $m$ and $n$ in \eqref{b_star} are in fact a $\delta$-pair.
\subsection{Optimization descents}
Here we highlight some useful properties of different optimization descents. 
The coefficient update \eqref{gsmo_update} leads to the following updates:
\begin{equation} \label{delta_f_update}
        \Delta f_{ij} \leftarrow \Delta f_{ij} + t K_{ij}, \quad
        \Delta f_{ij}^* \leftarrow \Delta f_{ij}^* + (t-s) K_{ij}^*,
\end{equation}
which follows from the substitution of \eqref{gsmo_update} into 
\eqref{oneclass_decision} and \eqref{oneclass_correcting}.
\begin{lemma}\label{lemma_1}
    Let's consider two cases:
    \begin{IEEEenumerate}
        \item The update 
        $\alpha_i \leftarrow \alpha_i + t$, 
        $\alpha_j \leftarrow \alpha_j - t$ 
        by \eqref{gsmo_objective_t}
        is applied to a pair $\alpha_i > 0$, $\alpha_j \geq 0$,
        which is $\tau$-violating according to \eqref{KKT_a}; \label{lemma_1_alpha}
        \item The update 
        $\delta_i \leftarrow \delta_i + s$, 
        $\delta_j \leftarrow \delta_j - s$ 
        by \eqref{gsmo_objective_s} is applied to a pair $\delta_i < 1$, $\delta_j > 0$,
        which is $\tau$-violating according to \eqref{KKT_d}. \label{lemma_1_delta}
    \end{IEEEenumerate}
    Then
    \begin{IEEEenumerate}
        \item[\mylabel{lemma_1_KK_pos}{(a)}] $K_{ij}+K_{ij}^*>0$ in the case \ref{lemma_1_alpha} and $K_{ij}^* > 0$ in the case~\ref{lemma_1_delta};
        \item[\mylabel{lemma_ts_opt}{(b)}]
        \begin{IEEEeqnarray}{rCl} 
            t &=-& \min\left\{\dfrac{\Delta f_{ij} + \Delta f_{ij}^*}{K_{ij}+K_{ij}^*}, \alpha_i\right\}, \label{lemma_1_t} \\
            s &=& \min\left\{\dfrac{\Delta f_{ij}^*}{K_{ij}^*}, 1-\delta_i, \delta_j\right\}; \label{lemma_1_s}
        \end{IEEEeqnarray}
        \item[\mylabel{lemma_1_ts_sign}{(c)}] $t < 0$; $s > 0$; 
        \item[\mylabel{lemma_1_not_violating}{(d)}] In both cases a pair becomes not a $\tau$-violating; 
        \item[\mylabel{lemma_1_Deltas}{(e)}] If the optimum point is not on-bound, then after the update we have $\Delta f_{ij} + \Delta f_{ij}^* = 0$ in the case \ref{lemma_1_alpha}, and $\Delta f_{ij}^* = 0$ in the case \ref{lemma_1_delta}.
    \end{IEEEenumerate}
\end{lemma}
Note that this lemma requires the violations to be simply positive $\Delta f_{ij} + \Delta f_{ij}^*>0$ and $\Delta f_{ij}^*>0$, whatever $\tau$, including $\tau=0$.
These violations ensure that $\varphi(t)$ and $\varphi^*(s)$ are always  quadratic (i.e. never linear) and their extrema are compared with only one of two constraint bounds.
The lemma also means that if the optimal point in \eqref{gsmo_objective_t} or \eqref{gsmo_objective_s} is not on-bound, then the violation
becomes zero.
{Also}, if {the optimal point} is on-bound, then {the violation} remains positive, but the coefficients become, respectively $\alpha_i=0$ and $\delta_i=1$ or $\delta_j=0$, and, therefore, the pair is not a $\tau$-violating.
Moreover, it follows from Lemma~\ref{lemma_1} and \eqref{delta_f_update} that every single-variable optimization \eqref{gsmo_objective_t} and \eqref{gsmo_objective_s} reduces both $\Delta f_{ij}+\Delta f_{ij}^*$ and $\Delta f_{ij}^*$.
It means that one descent "improves" both $\alpha$-KKT and $\delta$-KKT together, and, particularly, may remove both violations at once. 
Also, it can be shown that two-dimensional constraint optimization \eqref{gsmo_update}--\eqref{gsmo_alpha_delta} also removes violations,
because it can be represented as a sequence of $t$- and $s$-coordinate descents applied one after another and keeping $\alpha$- and $\delta$-pairs not $\tau$-violating.
\subsection{Optimization algorithm}
The above analysis leads to the following understanding of the SMO algorithm:
\begin{enumerate}
\item Find the most $\tau$-violating pair $i$, $j$ using \eqref{KKT_a} or \eqref{KKT_d}; 
\item Update $\alpha_i$,
$\alpha_j$ or $\delta_i$, $\delta_j$ with \eqref{lemma_1_t}--\eqref{lemma_1_s}; 
\item Recalculate $f(x_k)$, $f^*(x_k^*)$, $k=1,\ldots,l$,  to find the violating pairs at the next iteration.
\end{enumerate}
\begin{algorithm}
\SetAlgoCaptionSeparator{}
\DontPrintSemicolon
\SetKwInOut{Input}{Initialize}
\SetKwInOut{Input}{Input}
\Input{$\{x_k\}$, $\{x_k^*\}$; $K(x, y)$, $K^*(x^*, y^*)$; 
$\nu$, $\gamma$, $\tau$; $C$,~$C^*$.}
\SetKwInOut{Input}{Initialize}
\Input{
$\{\alpha_k\}$, $\{\delta_k\}$;\qquad
\mbox{\lFor(){$k \in C$}{calculate $f_k$ by \eqref{f_k}}}
\mbox{\lFor(){$k \in C^*$}{calculate $f_k^*$ by \eqref{f_k_star}}}
}
\While{$\mathrm{True}$}{
\tcp*[h]{violating pairs over caches}\;
$(i, j)$ = $\alpha\text{-pair}(C)$\\
$(m, n)$ = $\delta\text{-pair}(C^*)$\\
\If{$\Delta f_{ij}+\Delta f_{ij}^* > \tau$ \AlCapSty{\AlCapFnt and} 
     $\Delta f_{ij}+\Delta f_{ij}^* \geq \Delta f_{mn}^*$\quad}{$\alpha\text{-step}(i, j)$}
\ElseIf{$\Delta f_{mn}^* > \tau$ \AlCapSty{\AlCapFnt and} $\Delta f_{ij}+\Delta f_{ij}^* \leq \Delta f_{mn}^*$}{
\For(){$k \in \{m, n\}$}{
\leIf{$k\in C$}{$w_k=\mathrm{True}$}{$w_k=\mathrm{False}$}
}
$\delta\text{-step}(m, n)$\;
\For(){$k \in \{m, n\}$}{
\lIf{$w_k=\mathrm{False}$ \AlCapSty{\AlCapFnt and} $k\in C$}
    {recalculate $f_k$ by \eqref{f_k}}
}
}
\Else(\tcp*[h]{$\alpha$-pair over $C^*$}){
    \lFor(){$k\in {C^*\!\setminus\!C}$}{recalculate $f_k$ by \eqref{f_k}}
    $(i, j)$ = $\alpha\text{-pair}(C^*)$\;
    \If{$\Delta f_{ij}+\Delta f_{ij}^* > \tau$}{$\alpha\text{-step}(i, j)$}
    \Else(\tcp*[h]{violating pairs over all samples}){
        \lFor(){$k\notin {C^*}$}{recalculate $f_k, f_k^*$ by \eqref{f_k}--\eqref{f_k_star}}
        $(i, j)$ = $\alpha\text{-pair}(\{1,\ldots,l\})$\;
        $(m, n)$ = $\delta\text{-pair}(\{1,\ldots,l\})$\;
        \If{$\Delta f_{ij}+\Delta f_{ij}^* > \tau$ \AlCapSty{\AlCapFnt and} $\Delta f_{ij}+\Delta f_{ij}^* \geq \Delta f_{mn}^*$}{$\alpha\text{-step}(i, j)$}
        \ElseIf{$\Delta f_{mn}^* > \tau$ \AlCapSty{\AlCapFnt and} $\Delta f_{ij}+\Delta f_{ij}^* \leq \Delta f_{mn}^*$}{$\delta\text{-step}(m, n)$}
        \Else(\tcp*[h]{no violation pairs}){\AlCapSty{\AlCapFnt break}}
    }
}}
\SetKwInOut{Output}{Finalize}
\Output{Calculate $b^*$ and $\rho$ by \eqref{b_rho}--\eqref{b_star}.}
\SetKwInOut{Output}{Output}
\Output{$\{\alpha_k\}$, $\{\delta_k\}$, $\rho$, $b^*$.}
\renewcommand{\thealgocf}{}
\caption{SMO for OC-SVM+} \label{alg}
\end{algorithm}
\SetAlgoSkip{0em}
\setlength{\interspacetitleboxruled}{1pt}
\RestyleAlgo{boxruled}
\begin{function}[h]
    \DontPrintSemicolon
    $i=\argmax\limits_{\alpha_k > 0,\ k\in S}(f_k+f_k^*)$, \
    $j=\argmin\limits_{k \in S}(f_k+f_k^*)$\;
    \KwRet{$i$, $j$}
    \caption{$\alpha$-pair($S$)}
\end{function}
\begin{function}[h]
    \DontPrintSemicolon
    $m=\argmax\limits_{\delta_k < 1,\ k\in S}f_k^*$, \
    $n=\argmin\limits_{\delta_k > 0,\ k\in S}f_k^*$\; 
    \KwRet{$m$, $n$}
    \caption{$\delta$-pair($S$)}
\end{function}
\begin{procedure}[h]
    \DontPrintSemicolon
    calculate $t$ using \eqref{lemma_1_t}, \ set $s=0$\;
    $\alpha_i \leftarrow \alpha_i + t$, \quad $\alpha_j \leftarrow \alpha_j - t$\;
    \lFor(){$k\in C$}{update $f_k$ by \eqref{f_k_update}}
    \lFor(){$k\in C^*$}{update $f_k^*$ by \eqref{f_k_star_update}}
    \caption{$\alpha$-step($i$, $j$)} \label{alpha-step}
\end{procedure}
\begin{procedure}[h]
    \SetAlgoCaptionSeparator{}
    \DontPrintSemicolon
    calculate $s$ using \eqref{lemma_1_s}, \ set $t=0$\;
    $\delta_m \leftarrow \delta_m + s$, \quad $\delta_n \leftarrow \delta_n - s$\;
    \lFor(){$k\in C^*$}{update $f_k^*$ by \eqref{f_k_star_update} ($m$, $n$ instead of $i$, $j$)} 
    \caption{$\delta$-step($m$, $n$)} \label{delta-step}
\end{procedure}
\subsection{Initialization}
Similar to Sch{\"o}lkopf et al.~\cite{scholkopf2001estimating}, we start by setting the fraction $\nu$ of randomly selected coefficients to $\delta_k=1$, keeping the equality constraint in \eqref{dual_problem} satisfied:
the number $\nu l$ of $\delta_k$ are set to 1, 
and if $\nu l$ is not an integer, then one of $\delta_k$ is set to the value less than 1, such that $\sum_k \delta_k=\nu l$. 
So, according to \ref{prop1_nu_bound} in Proposition~\ref{proposition1},
the algorithm starts with extreme numbers of coefficients $\delta_k>0$ and $\delta_k=1$.
We initialize $\alpha_k$ in the same way, keeping $\sum_k \alpha_k=\nu l$.
\subsection{Working set selection}
Our algorithm finds both the $\alpha$- and the $\delta$-pair.
If the $\alpha$-pair is $\tau$-violating, 
then the step \eqref{gsmo_objective_t} is performed,
and if the $\delta$-pair is $\tau$-violating, then the step \eqref{gsmo_objective_s} is performed.
If pairs of both types are $\tau$-violating, then the one with the largest violation is chosen.

{\it {Generalized}} SMO for SVM+ described in Pechyony et al.~\cite{pechyony2010smo} first processes only $\delta$-pairs, and only when there are none, 
it looks for an $\alpha$-pair. 
For the selected pair, a two-variable subproblem is solved to update the {reducible} working set of four coefficients $\{\alpha_i,\ \alpha_j,\ \delta_i,\ \delta_j\}$.
Our choice of the best of two working sets 
$\{\alpha_i,\ \alpha_j\}$ and $\{\delta_m,\ \delta_n\}$ is essentially similar to Pechyony's et al.~\cite{pechyony2010smo, pechyony2011fast} {\it alternative SMO} algorithm for the {supervised} SVM+ and gives a significant speedup. 
\subsection{Updating and recalculating}
The algorithm operates with the values of the decision and the correcting functions without intercepts $\rho$ and $b^*$:
\begin{IEEEeqnarray}{rCl} 
    f_k &=& \frac{1}{\nu l}\sum\limits_{p=1}^l\alpha_p K(x_p, x_k), \label{f_k} \\
    f_k^* &=& \frac{1}{\gamma}\sum\limits_{p=1}^l(\alpha_p - \delta_p)K^*(x^*_p, x_k^*). \label{f_k_star}
\end{IEEEeqnarray}

Alternatively, we can just update these values on every iteration:
\begin{IEEEeqnarray}{rCl}
    f_k   &\leftarrow& f_k   + \dfrac{t}{\nu l} (K(x_i, x_k)-K(x_j, x_k)), \label{f_k_update}\\
    f_k^* &\leftarrow& f_k^* + \dfrac{t-s}{\gamma}(K^*(x_i^*, x_k^*)-K^*(x_j^*, x_k^*)). \label{f_k_star_update}
\end{IEEEeqnarray}
This follows from the substitution of \eqref{gsmo_update} into 
\eqref{f_k} and \eqref{f_k_star}.
The updating is faster than (re-)calculating, although it requires $f_k$ and $f_k^*$ to remain current, whereas recalculation does not.
\subsection{shrinking and Caching}
The values $f_k$ and $f_k^*$ are cached for a given subsets of data examples.
Here {\it caching} means that these values are stored in caches and do not need to be recalculated. 
Also, the fact, that only a subset of them are cached (usually not at-bound elements) means {\it shrinkage} (Joachims~\cite{joachims1998making}).
Caching of $f_k$ and $f_k^*$ should not be confused with the caching of kernels $K(x_i, x_j)$, $K^*(x_i^*, x_j^*)$ and the values $K_{ij}$, $K_{ij}^*$.
We introduce two different caches, denoted $C$ and $C^*$.
The cache $C$ contains the indices of data examples for which both $f_k$ and $f_k^*$ are currently known and are searched for the $\alpha$-pair.
The $C^*$ cache contains elements with known $f_k^*$ and the $\delta$-pair is searched for.
It naturally follows from this that $C\subseteq C^*$.

If there is no $\tau$-violated pair inside the caches, then the $\alpha$-pair inside the 
${C^*\!\setminus\!C}$ is first searched for, and if it is not found, then the pairs are searched outside the largest cache $C^*$.
The $f_k$ are updated for $k\in C$, and $f_k^*$ are updated for $k\in C^*$.
Outside caches, 
$f_k$ and $f_k^*$ are recalculated.
These heuristics of searching for the pairs from only a subset of examples and updating 
are similar to those used in other works on SMO for SVM including 
Platt's~\cite{platt1998fast} subset of not at-bound coefficients $\mathit{NB}$, Joachims's~\cite{joachims1998making} shrinking, as well as modifications of Keerthi et al.~\cite{keerthi2001improvements} and Bordes et al.~\cite{bordes2005fast}.

In our case, two types of dual coefficients with different sets of at-bound values give a variety of caching rules. 
The sets
\begin{align*}
  \mathit{NB} &= \{k~|~\alpha_k>0\ \mathrm{or}\ 0<\delta_k<1\},\\
  \mathit{NZ} &= \{k~|~\alpha_k>0\ \mathrm{or}\ \delta_k>0\},\\
  \mathit{ALL} &= \{1, \ldots, l\},
\end{align*}
are used in the following most typical combinations: 
$C=C^*=\mathit{NB}$; $C=C^*=\mathit{NZ}$;
$C=\mathit{NB}$, $C^*=\mathit{ALL}$; 
$C=C^*=\mathit{ALL}$ (no shrinkage).
In general, we consider $C$ and $C^*$ as unions of several of the following six disjoint sets
$\{k~|~\alpha_k>0,\ 0<\delta_k<1\}$,
$\{k~|~\alpha_k>0,\ \delta_k=0\}$,
$\{k~|~\alpha_k>0,\ \delta_k=1\}$,
$\{k~|~\alpha_k=0,\ 0<\delta_k<1\}$,
$\{k~|~\alpha_k=0,\ \delta_k=0\}$, and
$\{k~|~\alpha_k=0,\ \delta_k=1\}$, subject to $C\subseteq C^*$.
Any combination can be set through a hyperparameter.
  A small cache limits the choice of $\tau$-violating pair and leads to a lot of out-of-cache recalculations,
although a too wide cache may potentially require many updates \eqref{f_k_update}--\eqref{f_k_star_update} to be performed on elements, which coefficients are already stable and do not significantly affect iterations.
The experiments below compare different caches and show that choosing the optimal caching heuristic is very important.

Note also that the cache is a rule for an element to belong to it, and the set of cache elements itself usually changes after coefficient shift
$\alpha_i \leftarrow \alpha_i + t$, $\alpha_j \leftarrow \alpha_j - t$ or
$\delta_m \leftarrow \delta_m~+~s$, $\delta_n \leftarrow \delta_n - s$. 
For example, if the cache consists of not at-bound elements $\mathit{NB}$, then after the update some elements often become at-bound, i.e. $\alpha_k=0$, $\delta_k=0, 1$, and, therefore, do not belong to the cache.
Since $\alpha_k$ affects both $f_k$ and $f_k^*$, and $\delta_k$ affects only $f_k^*$, 
the $\alpha$-step requires updates to $f_k$ and $f_k^*$ whereas $\delta$-step only requires updating $f_k^*$. 
Also, the $\delta$-step can lead to the fact that $k$ can get into the 
$C$ cache without previously belonging to it ($w_k=\mathrm{False}$).
If it occurs, then $f_k$ needs to be recalculated.
\subsection{Modifications to consider}
Our selection of violating pair from the worst $\alpha$-pair and the worst $\delta$-pair follows from the first-order Taylor approximation of the objective \eqref{dual_problem} (see Bottou and Lin~\cite{bottou2007support}).
It could also be based on which of the quadratic objectives $\varphi(t)$ or $\varphi^*(s)$ gets the greatest gain, or to prefer a step, which removes both violations rather than only one.
Although, it needs calculation of the objectives in the former, and \eqref{lemma_1_t}, \eqref{lemma_1_s} and \eqref{delta_f_update} in the latter.
Anyway, all these methods should be justified or improved from first and second order approximations (Fan et al.~\cite{fan2005working}, Bottou and Lin~\cite{bottou2007support}).

It seems logical to use two-variable optimization in $\alpha$- and $\delta$-steps if the selected worst pair is also a $\tau$-violating by another criterion, 
i.e. for example, 
if for the $\alpha$-pair with $\Delta f_{ij}+\Delta f_{ij}^* > \tau$ there is also $\Delta f_{ij}^* > \tau$, or for the $\delta$-pair there is $\Delta f_{mn}^* > \tau$ together with $\Delta f_{mn}+\Delta f_{mn}^* > \tau$.
It would be reasonable to expect that two-variable optimization \eqref{gsmo_objective}--\eqref{gsmo_alpha_delta} may, in general,  offer better optimization gain than one-variable \eqref{gsmo_objective_t} and \eqref{gsmo_objective_s} applied subsequently.
However, this returns us back to reducible working set 
$\{\alpha_i, \alpha_j, \delta_i, \delta_j\}$ (or $\{\alpha_m, \alpha_n, \delta_m, \delta_n\}$), applied occasionally. 
Moreover, since the alternative violation is not necessarily the worst among pairs of its kind, i.e. among $\alpha$- or among $\delta$-violations, this modification needs more attention.
Although, as we mentioned above, such iteration removes both violations (i.e. one can show that \ref{lemma_1_not_violating} and \ref{lemma_1_Deltas} in Lemma~\ref{lemma_1} remain true), further study of the convergence is an open problem.

Note also that when the $\tau$-violating pair can not be found among the cached examples,
it could be searched out of the caches until the first $\tau$-violation occurs.
This would get rid of the recalculation of all out-of-cache $f_k$ and $f_k ^*$, 
but it would not yield the most $\tau$-violating pair.
Following Keerthi's et al.~\cite{keerthi2001improvements} principle always prefer the worst violation, we go through all out-of-cache elements.

We leave the study and experiments with these (somewhat questionable) modifications beyond the scope of this work.
\subsection{Convergence}
There are at least two methodologies for justifying the convergence of SMO algorithms for SVMs. 
The first one is by Keerthi and Gilbert~\cite{keerthi2002convergence} (with Takahashi and Nishi~\cite{takahashi2005rigorous}) and 
Lin~\cite{lin2001convergence, lin2002asymptotic, lin2002formal},
whereas the second one is by Bordes et al.~\cite{bordes2005fast}.
All of them use violation bound $\tau$ (or $\epsilon$), whereas the latter limits the size of the optimization step by additional small tolerance $\kappa>0$
(in our algorithm, this would mean that additional conditions are required to continue iterations: $-t > \kappa$, $s>\kappa$).
This leads to the concepts of $\kappa\tau$-violation and $\kappa\tau$-optimal solution (Bordes et al.~\cite{bordes2005fast}). 
Pechyony's et al.~\cite{pechyony2010smo} SMO {algorithms} for SVM+ are based on the latter approach,
for which $\kappa=0$ is an open problem.

In the former,
Keerthi and Gilbert proved the convergence of Platt's SMO for the supervised SVM in a finite number of steps when $\tau>0$.
Also, Lin~\cite{lin2001convergence, lin2002asymptotic} for the case $\tau=0$ (when the algorithm may not stop), proved its {so-called} asymptotic convergence,
that is if a subsequence of coefficient vectors converges, then it goes to the exact (not $\tau$-approximate) solution of SVM optimization problem.
Lin~\cite{lin2002formal} proved finite-time convergence for generalized SVM formulations.
Both Keerthi and Gilbert~\cite{keerthi2002convergence} and Lin~\cite{lin2002formal} have extensions to unsupervised one-class $\nu$-SVM.
We extend this former methodology to our unsupervised Sequential Minimal Optimization method for one-class Support Vector Machines with privileged information.
\subsection{Infinite-time convergence}
From the quadratic optimazation problems \eqref{gsmo_objective_t} and \eqref{gsmo_objective_s} we obtain the following lemma. 
\begin{lemma} \label{lemma_2}
    Let $\bm{\alpha}=(\alpha_1, \ldots, \alpha_l, \delta_1, \ldots, \delta_l)
    \in \mathbb{R}^{2l}$ is a vector containing a $\tau$-violating pair according 
    to \eqref{KKT_a} or \eqref{KKT_d} 
    and 
    $\bm\alpha_{new}\in \mathbb{R}^{2l}$ is a solution after an iteration in the SMO {for} OC-SVM+ algorithm. 
    Then, denoting the dual objective in \eqref{dual_problem} as $F(\bm\alpha)$, we have
    \begin{equation*}
        F(\bm\alpha) - F(\bm\alpha_\mathrm{new}) > \frac{\tau}{\sqrt{2}}||\bm\alpha_\mathrm{new}-\bm \alpha||.
    \end{equation*}
\end{lemma}
This lemma leads to the following proposition, which states that the algorithm converges to a limit solution in at least an infinite number of steps.
\begin{proposition} \label{prop_inf_convergence}
    Let $\bm\alpha(i)$ be the approximate solution after the $i$-th iteration of the algorithm. 
    Then there exists $\overline{\bm\alpha}=(\overline{\alpha}_1, \ldots, \overline{\alpha}_l, \overline{\delta}_1, \ldots, \overline{\delta}_l)$
    such that $\overline{\alpha}_k$ and $\overline{\delta}_k$ satisfy constraints in \eqref{dual_problem} and $\lim\limits_{i\to\infty}{\bm\alpha(i)}=\overline{\bm\alpha}$.
\end{proposition}
Note that Lemma~\ref{lemma_2} and Proposition~\ref{prop_inf_convergence} remain true if we change the way of pair selection.
It is only important that the optimization descent is applied only to a $\tau$-violating pair, where $\tau>0$ can be arbitrarily small.
\subsection{Finite-time convergence}
Next, we prove that the algorithm always stops within a finite number
of iterations, that is, the sequence $\{\bm\alpha(i)\}_{i=0}^{\infty}$ always converges
to a $\tau$-approximate solution of problem \eqref{dual_problem}.
\begin{theorem} \label{theorem_fin_convergence}
    SMO for OC-SVM+ algorithm stops in a finite number of iterations for any $\tau>0$.
\end{theorem}
Generally, the proof of the theorem (see Appendix~\ref{appendix3}) is similar to the proof of finite-time convergence of the Sequential Minimal Optimization for the supervised Support Vector Machines without privileged information proposed by Keerthi and Gilbert~\cite{keerthi2002convergence} and Takahashi and Nishi~\cite{takahashi2005rigorous}.
For this we introduce some notations and lemmas below.
Further in this section we use $k$ as the number of the iteration.
$(i(k), j(k))$ and $(m(k), n(k))$ are respectively $\alpha$- and $\delta$-pairs selected on $k$-th iteration (according to our algorithm, for every $k$ only one of two pairs is defined as a result of comparison of two pairs).
Let's denote 
$L_{\alpha}(p, q) = \{k\,|\,(i(k), j(k))=(p, q)\}$ and
$L_{\delta}(p, q) = \{k\,|\,(m(k), n(k))=(p, q)\}$
the sets of iterations, for which $(p, q)$ is selected as $\alpha$-pair and $\delta$-pair respectively. 
Also,  
$I_{\alpha}^{\infty}=\{(p, q):|L_{\alpha}(p, q)|=\infty\}$ and
$I_{\delta}^{\infty}=\{(p, q):|L_{\delta}(p, q)|=\infty\}$
are, respectively, the sets of $\alpha$- and $\delta$-pairs
selected at an infinite number of iterations.
We claim that $I_{\alpha}^{\infty}=I_{\delta}^{\infty}=\emptyset$.

All possible updates $\alpha_i^\mathrm{new} \leftarrow \alpha_i + t$, $\alpha_j^\mathrm{new} \leftarrow \alpha_j - t$ belong to the triangle 
$S_{\alpha}=\{(\alpha_i, \alpha_j)~|~\alpha_i \geq 0,\>
\alpha_j \geq 0, \> \alpha_i + \alpha_j \leq \nu l\}$,
i.e. $(\alpha_i, \alpha_j) \in S_{\alpha}$ and 
$(\alpha_i^\mathrm{new}, \alpha_j^\mathrm{new}) \in S_{\alpha}$ (Fig.~\ref{fig:alpha_steps}),
whereas the updates
$\delta_m^\mathrm{new} \leftarrow \delta_m + s$, 
$\delta_n^\mathrm{new} \leftarrow \delta_n - s$
belong to the rectangle
$S_{\delta}=\{(\delta_m, \delta_n)~|~0 \leq \delta_m \leq 1,\> 0 \leq \delta_n \leq 1\}$,
i.e. 
$(\delta_m, \delta_n) \in S_{\delta}$ and 
$(\delta_m^\mathrm{new}, \delta_n^\mathrm{new}) \in S_{\delta}$ (Fig.~\ref{fig:delta_steps}).
The areas have the interior parts $\mathrm{int}S_{\alpha}$, $\mathrm{int}S_{\delta}$, 
the edges 
$S_{\alpha}^s=\{(\alpha_i, \alpha_j)~|~0<\alpha_i<\nu l,\ \alpha_j=0\}$,
$S_{\alpha}^w=\{(\alpha_i, \alpha_j)~|~\alpha_i=0,\ 0<\alpha_j<\nu l\}$,
$S_{\alpha}^{ne}=\{(\alpha_i, \alpha_j)~|~0<\alpha_i<\nu l,\ \alpha_i+\alpha_j=\nu l\}$,
and 
$S_{\delta}^n=\{(\delta_m, \delta_n)~|~0<\delta_m<\nu l,\ \delta_n=1\}$,
$S_{\delta}^s=\{(\delta_m, \delta_n)~|~0<\delta_m<\nu l,\ \delta_n=0\}$,
$S_{\delta}^w=\{(\delta_m, \delta_n)~|~\delta_m=0,\ 0<\delta_n<\nu l\}$,
$S_{\delta}^e=\{(\delta_m, \delta_n)~|~\delta_m=1,\ 0<\delta_n<\nu l\}$,
as well as the corners $(0, 0)$, $(\nu l, 0)$, and $(0, \nu l)$ in $S_{\alpha}$
and $(0, 0)$, $(1, 0)$, $(0, 1)$, and $(1, 1)$ in $S_{\delta}$.
The area where $\tau$-violations can occur includes $\mathrm{int}S_{\alpha}$, the edges $S_{\alpha}^s$, $S_{\alpha}^{ne}$ and the corner $(\nu l, 0)$ in $S_{\alpha}$, and 
$\mathrm{int}S_{\delta}$, the edges $S_{\delta}^w$, $S_{\delta}^n$ and the corner $(0, 1)$ in $S_{\delta}$.
Studying these steps using the convergence to a limit $\overline{\bm\alpha}$ by Proposition~\ref{prop_inf_convergence} and 
\ref{lemma_1_not_violating}, \ref{lemma_1_Deltas} in Lemma~\ref{lemma_1} leads to the following lemma.
\setlength{\unitlength}{\linewidth/24}
\thicklines
\newcommand{\pW}{12}
\newcommand{\pH}{12}
\newcommand{\aoW}{2}
\newcommand{\aoH}{1}
\newlength{\scl}
\settowidth{\scl}{$W$}
\setlength{\scl}{\scl/\pW}
\newlength{\dcl}
\settowidth{\dcl}{$ww$}
\setlength{\dcl}{\dcl/\pW}
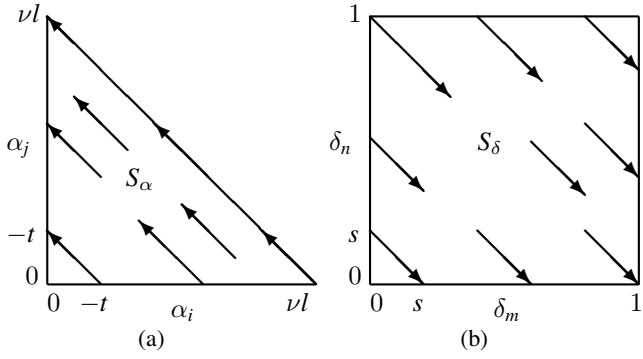
\begin{figure}
\begin{subfigure}{0.49\linewidth}
\begin{picture}(\pW, \pH)
\put(\aoW, \aoH){\line(1, 0){\fpeval{\pW-\aoW}}}
\put(\aoW, \fpeval{\pH-\aoH}){\line(1, -1){\fpeval{\pW-\aoW}}}
\put(\aoW, \fpeval{\pH-\aoH}){\line(0, -1){\fpeval{\pH-2*\aoH}}}
\put(\fpeval{\aoW}, 0){$0$}
\put(\fpeval{2*\aoW-\scl}, 0){$-t$}
\put(\fpeval{\pW-\dcl}, 0){$\nu l$}
\put(\fpeval{\aoW/2+\pW/2-\scl/2}, -\fpeval{\aoH/5}){$\alpha_i$}
\put(\fpeval{\aoW-\scl}, \aoH){$0$}
\put(\fpeval{\aoW-\dcl-\scl/2}, \fpeval{\aoH+\aoW*8/10}){$-t$}
\put(\fpeval{\aoW-\dcl}, \fpeval{\pH-1.2*\aoH}){$\nu l$}
\put(\fpeval{\aoW-\dcl-\scl/2}, \fpeval{\pH/2}){$\alpha_j$}
\put(\fpeval{2*\aoW}, \aoH){\vector(-1, 1){\fpeval{\aoW}}}
\put(\fpeval{\aoW*9/10+\pW/2}, \aoH){\vector(-1, 1){\fpeval{1.2*\aoW}}}
\put(\fpeval{\aoW*3/2+\pW/2}, \fpeval{\aoH*2}){\vector(-1, 1){\fpeval{\aoW}}}
\put(\fpeval{\aoW*3/2+\pW/2-\aoH*4}, \fpeval{\aoH*2+\aoH*4}){\vector(-1, 1){\fpeval{\aoW}}}
\put(\fpeval{2*\aoW}, \fpeval{\pH/2-\aoH}){\vector(-1, 1){\fpeval{\aoW}}}
\put(\pW, \aoH){\vector(-1, 1){\fpeval{\aoW}}}
\put(\fpeval{\pW-4*\aoH}, \fpeval{\aoH + 4*\aoH}){\vector(-1, 1){\fpeval{\aoW}}}
\put(\fpeval{\aoW+\aoH}, \fpeval{\pH-2*\aoH}){\vector(-1, 1){\fpeval{\aoH}}}
\put(\fpeval{\pW*41/100}, \fpeval{\pH*39/100}){$S_{\alpha}$}
\end{picture}
\caption{}
\label{fig:alpha_steps}
\end{subfigure}
\begin{subfigure}{0.49\linewidth}
\begin{picture}(\pW, \pH)
\put(\aoW, \aoH){\line(1, 0){\fpeval{\pW-\aoW}}}
\put(\fpeval{\pW}, \aoH){\line(0, 1){\fpeval{\pH-2*\aoH}}}
\put(\fpeval{\pW}, \fpeval{\pH-\aoH}){\line(-1, 0){\fpeval{\pW-\aoW}}}
\put(\aoW, \fpeval{\pH-\aoH}){\line(0, -1){\fpeval{\pH-2*\aoH}}}
\put(\aoW, 0){$0$}
\put(\fpeval{2*\aoW-\scl/2}, 0){$s$}
\put(\fpeval{\pW}, 0){\!\!$1$}
\put(\fpeval{\aoW/2+\pW/2-\scl/2}, -\fpeval{\aoH/5}){$\delta_m$}
\put(\fpeval{\aoW-\scl}, \aoH){$0$}
\put(\fpeval{\aoW-\scl}, \fpeval{\aoH+\aoW*8/10}){$s$}
\put(\fpeval{\aoW-\scl}, \fpeval{\pH-1.2*\aoH}){$1$}
\put(\fpeval{\aoW-\dcl-\scl/2}, \fpeval{\pH/2}){$\delta_n$}
\put(\aoW, \fpeval{\aoW+\aoH}){\vector(1, -1){\aoW}}
\put(\aoW, \fpeval{\aoH/2+\pH/2}){\vector(1, -1){\fpeval{2*\aoH}}}
\put(\aoW, \fpeval{\pH-\aoH}){\vector(1, -1){\fpeval{1.5*\aoW}}}
\put(\fpeval{\pW-\aoW}, \fpeval{\pH-\aoH}){\vector(1, -1){\fpeval{\aoW}}}
\put(\fpeval{\pW/2}, \fpeval{\pH-\aoH}){\vector(1, -1){\fpeval{1.2*\aoW}}}
\put(\fpeval{\pW-\aoW}, \fpeval{\aoH+\aoW}){\vector(1, -1){\fpeval{\aoW}}}
\put(\fpeval{\pW-\aoW}, \fpeval{\pH/2+\aoW/2}){\vector(1, -1){\fpeval{\aoW}}}
\put(\fpeval{\pW/2}, \fpeval{\aoH+\aoW}){\vector(1, -1){\fpeval{\aoW}}}
\put(\fpeval{\pW/2+\aoW}, \fpeval{\pH/2+\aoH/3}){\vector(1, -1){\fpeval{\aoW}}}
\put(\fpeval{\pW/2}, \fpeval{\pH/2}){$S_{\delta}$}
\end{picture}
\caption{}
\label{fig:delta_steps}
\end{subfigure}
\caption{$\alpha$- and $\delta$-coefficient updates.}
\end{figure}
\begin{lemma} \label{lemma_fin_convergence}
    Let $(p, q)\in I_{\alpha}^{\infty} (I_{\delta}^{\infty})$.
    Then there exists a $\overline{k}$ such that at all iterations $k \geq \overline{k}$, 
    for which $(i(k), j(k)) = (p, q)$ 
    ($(m(k), n(k)) = (p, q)$) we have:
    \begin{IEEEenumerate}
        \item[\mylabel{lemma_fin_convergence_to_int}{(a)}] 
        $\alpha_p^\mathrm{new} = 0$ 
        ($\delta_p^\mathrm{new}=1$ or $\delta_q^\mathrm{new}=0$);
        \item[\mylabel{lemma_fin_convergence_from_int}{(b)}] 
        $\alpha_q = 0$ ($\delta_p=0$ or $\delta_q=1$); 
        \item[\mylabel{lemma_fin_convergence_limits}{(c)}] 
        $\alpha$-pair ($\delta$-pair) can move only from $(\alpha_p, 0)$ to 
        $(0, \alpha_q^\mathrm{new})$, $\alpha_p=\alpha_q^\mathrm{new}>0$,
        converging to $(\overline{\alpha}_p, \overline{\alpha}_q)=(0, 0)$
        (from $(0, \delta_q)$ to $(\delta_p^\mathrm{new}, 0)$,
        $0 < \delta_q=\delta_p^\mathrm{new} < 1$, converging to $(\overline{\delta}_p, \overline{\delta}_q)=(0, 0)$ or from $(\delta_p, 1)$ to $(1, \delta_q^\mathrm{new})$,
        $0<\delta_p = \delta_q^\mathrm{new} < 1$,
        converging to $(\overline{\delta}_p, \overline{\delta}_q)=(1, 1)$  );
    \end{IEEEenumerate}
\end{lemma}
According to Lemma~\ref{lemma_fin_convergence}, 
if a pair $(p, q)$ remains $\tau$-violating infinitely many times,
then by \ref{lemma_fin_convergence_to_int} it can not move to and by \ref{lemma_fin_convergence_from_int}
it can not move from not on-bound point 
(i.e. a point from $\mathrm{int}S_{\alpha}\cup S_{\alpha}^{ne}$ or $\mathrm{int}S_{\delta}$) 
infinitely many times.
Using \ref{lemma_fin_convergence_to_int} and 
\ref{lemma_fin_convergence_from_int} we obtain \ref{lemma_fin_convergence_limits},
which also concludes that the limit point 
$(\overline{\alpha}_p, \overline{\alpha}_q)$ 
($(\overline{\delta}_p, \overline{\delta}_q)$)
can only be at the some of not a $\tau$-violating corners,
and after sufficiently large number of iterations, 
the pair moves only from one adjacent edge to another. 
Moreover, \ref{lemma_fin_convergence_limits} leads to 
\ref{lemma_fin_convergence_rs} of the of the following lemma.
\begin{lemma} \label{lemma_fin_convergence2}
    Let $(p, q)\in I_{\alpha}^{\infty} (I_{\delta}^{\infty})$.
    Then
    \begin{IEEEenumerate}
        \item[\mylabel{lemma_fin_convergence_qp}{(a)}] 
        $(q, p) \notin I_{\alpha}^{\infty} (I_{\delta}^{\infty})$;
        \item[\mylabel{lemma_fin_convergence_rs}{(b)}]
        There must be $r\neq q$ and $s\neq p$ such that
        $(r, p)\in I_{\alpha}^{\infty} (I_{\delta}^{\infty})$ and 
        $(q, s)\in I_{\alpha}^{\infty} (I_{\delta}^{\infty})$.
    \end{IEEEenumerate}
\end{lemma}
To return on the edge from which the pair $(p, q)$ has moved,
$p$ and $q$ need to form pairs with other examples at some of the subsequent iterations.
Studying this leads to \ref{lemma_fin_convergence_rs}
of Lemma~\ref{lemma_fin_convergence2},
which {ultimately comes up with a desired} contradiction in the proof of Theorem~\ref{theorem_fin_convergence}.
\section{Experiments} \label{sec_experiments}
\subsection{Synthetic data}
Let's start by demonstrating how OC-SVM+ works with simple synthetic data.
We use a random sample from two-dimensional standard normal distribution and two outliers $(-2.5, 0)$ and $(0, 2.5)$. 
For the data examples to the right from the origin we generate privileged feature that is zero-mean normal variate with standard deviation $0.1$, whereas for the examples to the left we generate values uniformly on $(-100, -10)\cup(10, 100)$.
Since privileged feature characterizes left examples as anomalies, the domain is shifted to the right (Fig.~\ref{fig:expr1_ocsvm_plus}).
This shows that knowledge about anomalies is transferred from the privileged feature to the domain of the original (decision) features. As a result, many left examples that are not detected as anomalies without privileged information also become anomalies, i.e. they are outside the domain boundary.

Fig.~\ref{fig:expr1_ocsvm_plus}, \ref{fig:expr1_ocsvm_plus_gamma}, and \ref{fig:expr1_ocsvm_plus_nu} 
illustrate Proposition~\ref{proposition1} as follows.
\ref{prop1_d1} means that all domain-outsiders are orange in color (with or without red squares, i.e. are or are not support vectors).
\ref{prop1_3cond} means that green examples inside red squares can be only on-bound.
Also, on-bound support vectors are green, and on-bound green examples are the support vectors (some non-green support vectors lie near the boundary but not on it).
\ref{prop1_nu_bound} means that the fraction of green and orange examples is not less that $\nu$,
and the fraction of orange examples is not greater than $\nu$ (the latter is not so obvious in the drawing since all the orange examples are drawn on top of all the others).
\ref{prop1_nu_bound_anomaly} means that the fraction of domain-outsiders is not greater than $\nu$.

The influence of privileged information is controlled by two hyperparameters: the regularization $\gamma$ and the privileged kernel parameter $\sigma^*$.
As expected, the lower these values, the stronger the effect (Fig.~\ref{fig:expr1_recall}).
Note, that when weakening the influence of privileged information, i.e. increasing the regularization $\gamma$, the domain remains almost unmoved (Fig.~\ref{fig:expr1_ocsvm_plus_gamma}), but different from the domain {of simple OC-SVM without privileged information} (Fig.~\ref{fig:expr1_ocsvm}). 
The latter demonstrates the fact that, although the privileged term in \eqref{dual_problem} vanishes for large $\gamma$, the problem does not have OC-SVM's condition $0\leq\alpha_k\leq 1$.
Similar to OC-SVM {without privileged information}, the size of the domain area decreases and the two outliers become anomalies if we increase $\nu$ (Fig.~\ref{fig:expr1_ocsvm_nu}, \ref{fig:expr1_ocsvm_plus_nu}).
\subsection{Performance on real data}
\begin{figure*}[t]
\centering
\begin{subfigure}{0.33\textwidth}
  \includegraphics[width=\linewidth]{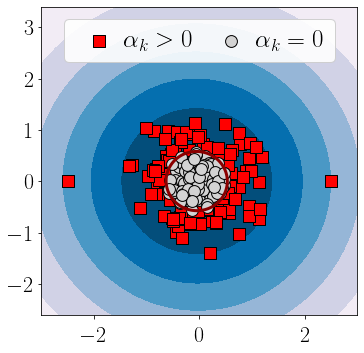}
  \caption{\ }
  \label{fig:expr1_ocsvm}
\end{subfigure}\hfil
\begin{subfigure}{0.33\textwidth}
  \includegraphics[width=\linewidth]{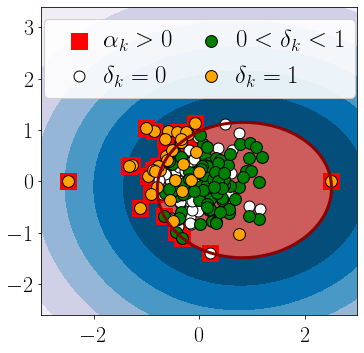} 
  \caption{\ }
  \label{fig:expr1_ocsvm_plus}
\end{subfigure}\hfil
\begin{subfigure}{0.33\textwidth}
  \includegraphics[width=\linewidth]{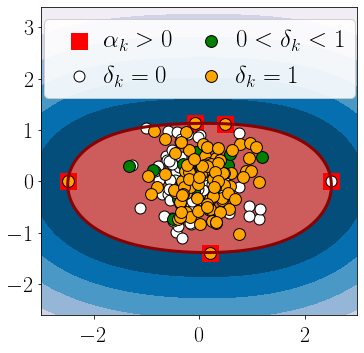}
  \caption{\ }
  \label{fig:expr1_ocsvm_plus_gamma}
\end{subfigure}
\\
\begin{subfigure}{0.33\textwidth}
  \includegraphics[width=\linewidth]{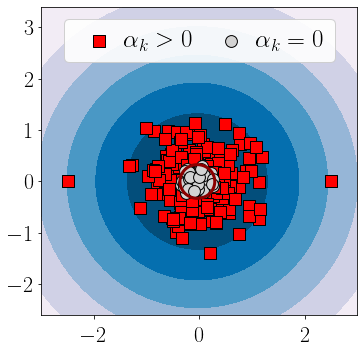}
  \caption{\ }
  \label{fig:expr1_ocsvm_nu}
\end{subfigure}\hfil
\begin{subfigure}{0.33\textwidth}
  \includegraphics[width=\linewidth]{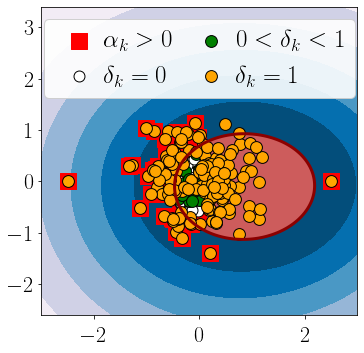}
  \caption{\ }
  \label{fig:expr1_ocsvm_plus_nu}
\end{subfigure}\hfil
\begin{subfigure}{0.33\textwidth}
  \includegraphics[width=\linewidth]{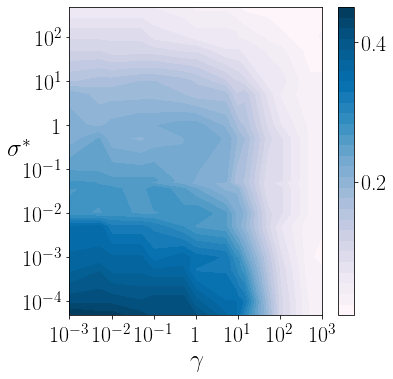}
  \caption{\ }
  \label{fig:expr1_recall}
\end{subfigure}
\caption{
\eqref{fig:expr1_ocsvm} -- OC-SVM,  $\nu=0.5$; 
\eqref{fig:expr1_ocsvm_plus} -- OC-SVM+, $\nu=0.5$, $\gamma=10^{-3}/{\nu l}$;
\eqref{fig:expr1_ocsvm_plus_gamma} -- OC-SVM+, $\nu=0.5$, $\gamma=10^3/{\nu l}$;
\eqref{fig:expr1_ocsvm_nu} -- OC-SVM, $\nu=0.8$;
\eqref{fig:expr1_ocsvm_plus_nu} -- OC-SVM+, $\nu=0.8$, $\gamma=10^{-3}/{\nu l}$;
\eqref{fig:expr1_recall} -- Anomaly detection recall (fraction of anomalies among all points to the left of the origin) depending on privileged hyperparameters.}
\end{figure*}
In the next experiment, we use the Statlog (Shuttle) data set from the Outlier Detection DataSets library (Shebuti~\cite{Rayana2016ODDS}).
The original version of Shuttle data is in the UCI machine learning repository (Dua and Graff~\cite{Dua2019UCI}). 
The data set consists of 49097 feature vectors of dimension 9 and anomaly label.
Features 0, 6, 7, and 8, most relevant to the anomaly flag, are selected as privileged information, features 2 and 4 are the original, and features 1, 3 and 5, the least associated with anomaly, are discarded.
All features are normalized using a standard scaler, i.e. we subtract the mean and divide by the standard deviation.
Then, we prepare {ten} train-test splits, 
such that training sets do not overlap with each other, 
and each of them has $10\%$  anomalies.

OC-SVM+ is trained on original and privileged features and evaluated on only original features.
In addition, two simple OC-SVMs without privileged information are used for comparison: one using only the original features and the other using all features.  
The outputs of the decision functions for the two OC-SVMs and OC-SVM+ are compared with the anomaly labels using average precision score (Zhu~\cite{zhu2004recall}).
This metric approximates the area under precision-recall curve,
which is a typical metric choice due to data imbalance in anomaly detection.
Given a fixed $\sigma$ of the original kernel and tolerance $\tau=10^{-3}$, other hyperparameters including $\nu$ (in all models), as well as $\gamma$ and $\sigma^*$ are optimized with cross-validation and grid search method.
The regions for searching for hyperparameters are chosen large enough so that their optimal values are within region bounds.
As expected, the optimal score for OC-SVM+ is between the optimal scores for OC-SVMs, noticeably outperforming OC-SVM trained on only original features (Fig.~\ref{fig:shuttle}).
\begin{figure}[t]
  \centering
  \includegraphics[width=\linewidth]{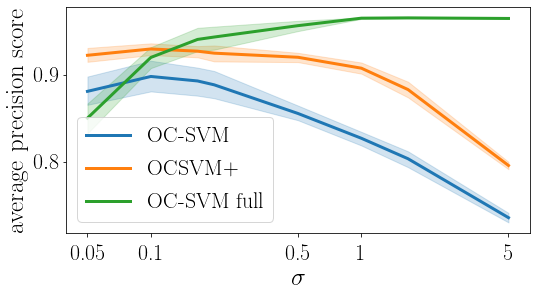}
  \caption{Anomaly detection average precision score with 0.95 confidence interval: OC-SVM with original features, OC-SVM+ and OC-SVM with full feature set.}
  \label{fig:shuttle}  
\end{figure}
\subsection{Speed tests on real data}
\begin{figure*}[h]
  \centering
  \includegraphics[width=\linewidth]{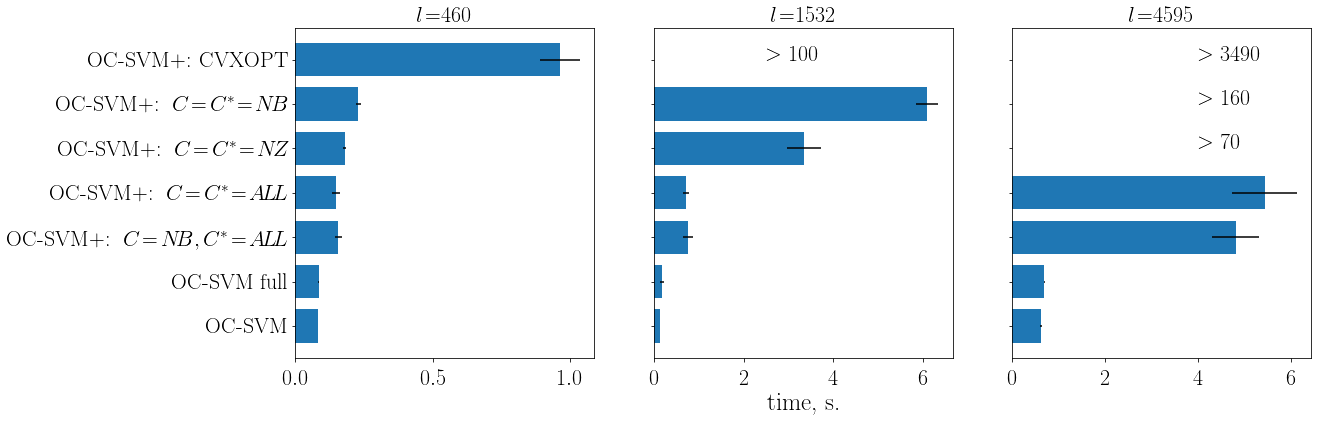}
  \caption{Comparison of processing times of OC-SVMs (original features and all features), OC-SVM+ via SMO with different caching, and OC-SVM+ via CVXOPT.}
  \label{fig:shuttle_time}  
\end{figure*}
In the speed comparison experiments we run our SMO for OC-SVM+ versus OC-SVMs with original and full feature sets, as well as versus {offline (batch)} quadratic optimization solver, applied directly to \eqref{dual_problem}.
For the latter, we use the package CVXOPT (Vandenberghe~\cite{Vandenberghe2010CVXOPT}), based on interior-point methods for large-scale cone programming (Andersen et al.~\cite{Andersen2011Interior}).
Also, we compare the performance of SMO for OC-SVM+ with different caching options $C$ and $C^*$.
We use the optimal hyperparameters $\nu$, $\sigma$, $\gamma$, and $\sigma^*$ found after tuning OC-SVM+ using the grid search method from the experiment above.
All OC-SVMs+ and two OC-SVMs have the same $\nu$ and $\sigma$.
We train every model 128 times for ten different reduced sample sizes $l$.
The means and standard deviations of model training times for the first three sample sizes are shown in Fig.~\ref{fig:shuttle_time}.
For the extent of privileged regularization to be the same for all sample sizes, i.e. when the ratio between original coefficient $\nu l$ and privileged $\gamma$ in \eqref{dual_problem} does not depend on $l$, we set $\gamma=\gamma_0 l$, where $\gamma_0$ is the same for all reduced $l$.
$\gamma_0$ is found as $\gamma/l$ for the optimal $\gamma$ and the size $l$ taken from the experiment above.
In all experiments, the kernel caches are large enough that every of 
$K(x_i, x_j)$, $K^*(x_i^*, x_j^*)$ and $K_{ij}$, $K_{ij}^*$
is calculated only once. 
For every sample size $l$ and every pair of models 
we compare the training times using Mann-Whitney test.
The choice of this nonparametric test instead of ANOVA is due to the fact that different 128-element groups have different standard deviations and the distributions are not close to normal.

The test showed that starting from the third sample size ($l=4595$, right in Fig.~\ref{fig:shuttle_time}) each model has a statistically larger mean training time than all models depicted below it (all p\_values are less than even $10^{-8}$).
So, SMO essentially outperforms the existing CVXOPT especially when the sample size $l$ grows. 
Also, the experiments show that the performance significantly depends on the choice of caching heuristics $C$ and $C^*$. 
It follows from the results shown that, generally, frequent updates of most of $f_k$, $f_k^*$ due to large caches $C=\mathit{NB}$, $C^*=\mathit{ALL}$ or $C=C^*=\mathit{ALL}$ work faster than occasional recalculations when the caches are small as, for example, cover only non-zero coefficients $C=C^*=\mathit{NZ}$. 
This is especially true when $l$ is large.
Although, it seems reasonable to allow recalculations of some $f_k$  by setting $C=\mathit{NB}$, $C^*=\mathit{ALL}$, since this can demonstrate  slightly better {(but statistically significant)} results than updating all of $f_k$, $f_k^*$ at each iteration, i.e. $C=C^*=\mathit{ALL}$.
Therefore, when the sample size becomes large enough the picture stabilizes and $C=\mathit{NB}$, $C^*=\mathit{ALL}$ becomes the best caching.
Although, no-shrinking option $C=C^*=\mathit{ALL}$ is only slightly behind.
The latter is, to some extent, in tune with training without shrinking (Chang and Lin~\cite{Chang2011LIBSVM}) if kernel caches are not full or $\tau$ is large. 
SMO algorithms for {Support Vector Machines} without privileged information in 
Platt~\cite{platt1998fast}, Keerthi et al.~\cite{keerthi2001improvements}, and Sch{\"o}lkopf et al.~\cite{scholkopf2001estimating}, in fact, use the heuristic $C=\mathit{NB}$, where $\mathit{NB} = \{k~|~0<\alpha_k<Const\}$. 
However, it is also worth noting that a natural generalization of this approach to $C=C^*=\mathit{NB}$, $\mathit{NB} = \{k~|~\alpha_k>0\ \mathrm{or}\ 0<\delta_k<1\}$, in OC-SVM+ will not be a good choice of caching rule.
It is even slower than $C=C^*=\mathit{NZ}$ ({Fig.}~\ref{fig:shuttle_time}), because smaller caches, 
$\mathit{NB}\subseteq\mathit{NZ}$, lead to larger number of out-of-cache recalculations. 
\section{Discussion and future work} \label{sec_discussion}
We have achieved the main goal of the study --- SMO algorithm for the OC-SVM model is significantly faster than the universal solver of the quadratic programming problem with given constraints.
It was also expected that the $\tau$-approximate solution is reached in a finite number of iterations.
The proof is rather complicated, but is fully given with all details.
In the future, it would be interesting to improve our method
by selection of violating pair with second-order approximations (Fan et al.~\cite{fan2005working}, Glasmachers et al.~\cite{glasmachers2006maximum}) and 
by the use of conjugate directions similar to Pechyony and Vapnik~\cite{pechyony2011fast}.
We are also working on a more detailed analysis of caching (and shrinking), in particular
the relationship between time performance and shrinking with different tolerances $\tau$.
Of particular interest is the  application of $\kappa\tau$-methodology of Bordes et al.~\cite{bordes2005fast} in order to establish the convergence of the $\kappa\tau$-optimal solution to the theoretical solution of the problem \eqref{dual_problem}.

The performance of the algorithms also depends on their implementation, i.e., for example, an implementation in C++ or Cython is faster than in Python. 
Therefore, to correctly compare the performance of different algorithms, it is preferable to have either low-level implementations of both, or implementations in the same environment (for example, Python or R). 
The degree of optimization of the program code also matters. 
Perhaps for this reason, the implementations of aSMO and gSMO algorithms in Pechyony et al.~\cite{pechyony2010smo} are not faster in time than the offline solver, although they gain significantly in complexity because they are online methods. 
At the same time, our implementation of OC-SVM+ in Cython allowed comparison even with the professional LIBSVM.
As expected, our SMO for OC-SVM+ takes longer to converge than SMO for OC-SVM without privileged information taken from the LIBSVM library for comparison.
Also, in our implementation of SMO all at-bound coefficients $\alpha_k=0$, $\delta_k=0, 1$,
are tracked using special vectors of boolean labels to avoid floating point number comparison.

As we noted, the OC-SVM+ model has already been formulated and used in other works. 
Our main focus has been on its fast training method rather than its properties (although we present Proposition~\ref{proposition1} and the first two experiments that do not characterize the SMO algorithm).
We have paid considerable attention to SMO's caching policies (rules).
We have shown that the choice of caching can strongly affect performance, and that the variants $C=C^*=\mathit{NB}$ and $C=C^*=\mathit{NZ}$ obtained by a simple generalization of the caching heuristics used in SMO for supervised SVM can actually be the slowest. 
An experiment with a single dataset is sufficient for this conclusion.
However, we cannot say with certainty that caching $C=\mathit{NB}$, $C^*=\mathit{ALL}$ is the best in most cases. For this purpose, experiments with a large number of datasets and different number of privileged features and anomalies in their values are planned in the future. It is also necessary to compare caching rules for large enough datasets when values $K(x_i, x_j)$, $K^*(x_i^*, x_j^*)$, and $K_{ij}$, $K_{ij}^*$ are recalculated due to impossibility to store everything in memory. Such larger benchmark will also allow us to compare the OC-SVM+ model itself with other anomaly detection tools for privileged information.
\section{Conclusion} \label{sec_conclusion}
We have incorporated LUPI paradigm into well known {one-class} SVM ($\nu$-SVM) model for anomaly detection and developed the Sequential Minimal Optimization method for it.
Our work is the full research pipeline: problem statement, studying the properties of the solution of the optimization problem, optimization algorithm, proof of convergence, as well as experiments both demonstrating and benchmark tests.

The unsupervised one-class SVM model with privileged features is a universal tool for anomaly detection, i.e. it can handle any tabular data without missing values. The model works with any given kernel.
The proposed SMO algorithm has allowed training this model much faster than existing methods.
The OC-SVM+ model provides a good baseline against which to compare future anomaly detection methods under privileged features.
\begin{appendices}
\section{{Abbreviations}} \label{appendix0}
\begin{table}[h!]
\caption{{\textbf{List of abbreviations}}}
\label{table1}
\renewcommand{\arraystretch}{1.5}
\centering
\begin{tabular}{|p{0.3\linewidth}|p{0.6\linewidth}|}
\hline
Abbreviattion& Definition\\
\hline
CVXOPT&Software package for Convex Optimization methods\\
gSMO&Generalized Sequential Minimal Optimization algorithm\\
KKT&Karush-Kuhn-Tucker conditions\\
LIBSVM&Software library for the family of Support Vector Machine algorithms\\
LUPI&Learning Using Privileged Information paradigm\\
OC-SVM ($\nu$-SVM) &One-class Support Vector Machine algorithm\\
OC-SVM+&One-class Support Vector Machine with privileged information algorithm\\
OCSVM-PLUS&Software library for one-class Support Vector Machine with privileged information algorithm\\
PI& Privileged Information\\
SMO&Sequential Minimal Optimization algorithm\\
SSVM& Structured Support Vector Machine algorithm\\
SVDD&Support Vector Domain Description model\\
SVM&Support Vector Machine algorithm\\
SVM+&Support Vector Machine with privileged information algorithm\\
\hline
\end{tabular}
\end{table}
\section{Dual problem formulation} \label{appendix1}
The problem \eqref{primal_problem} has the Lagrangian with the coefficients $\alpha_k \geq 0$, $\beta_k \geq 0$, $\mu_k \geq 0$,
\begin{IEEEeqnarray*}{rCl}
    L &=& \frac{\nu l}{2} \|w\|^2 + \frac{\gamma}{2} \|w^*\|^2 - \nu l \rho  \\
    &&+ \sum_k \left[ (w^*\cdot\phi^*(x_k^*)) + b^* + \zeta_k \right] \\
    &&-\sum_k \alpha_k \left[ (w\cdot\phi(x_k))  - \rho + (w^*\cdot\phi^*(x_k^*)) + b^*\right]\\
    &&-\sum_k \beta_k \left[ (w^*\cdot\phi^*(x_k^*)) + b^* + \zeta_k \right]\\
    &&-\sum_k \mu_k \zeta_k.
\end{IEEEeqnarray*}
Setting partial derivatives with respect to the primal variables $w$, $w^*$, $\rho$, $b^*$, $\zeta$ to zero
and introducing $\delta_k = 1-\beta_k \leq 1$, we obtain
\begin{IEEEeqnarray}{rClCl} 
\frac{\partial L}{\partial w} &=&\IEEEeqnarraymulticol{3}{l}{\nu l w -\sum_k \alpha_k\phi(x_k) = 0} \nonumber \\
 &&\IEEEeqnarraymulticol{3}{l}{\qquad \Rightarrow \quad w = \frac{1}{\nu l } \sum_k\alpha_k\phi(x_k),} \label{w_opt}\\
\frac{\partial L}{\partial w^*}  &=& \IEEEeqnarraymulticol{3}{l}{\gamma w^* - \sum_k (\alpha_k - \beta_k + 1)\phi^*(x_k^*) = 0} \nonumber \\
 &&\IEEEeqnarraymulticol{3}{l}{\qquad \Rightarrow \quad w^* = \frac{1}{\gamma}\sum_k (\alpha_k - \delta_k) \phi^*(x_k^*),}  \label{w_star_opt}\\
\frac{\partial L}{\partial \rho} &=& -\nu l + \sum_k\alpha_k = 0 & \Rightarrow \quad & \sum_k \alpha_k = \nu l, \label{alpha_opt}\\
\frac{\partial L}{\partial b^*} &=& \sum_k 1 - \beta_k - \alpha_k = 0 \nonumber \\
 &&\IEEEeqnarraymulticol{3}{l}{\qquad \Rightarrow \quad \sum_k \delta_k = \sum_k\alpha_k,} \label{delta_opt} \\
\frac{\partial L}{\partial \zeta_k}  &=& 1 - \beta_k - \mu_k = 0 & \Rightarrow \quad & 0 \leq \delta_k \leq 1. \label{delta_opt_limint}
\end{IEEEeqnarray}

Using \eqref{alpha_opt} and substitute the optimal $w$ according to \eqref{w_opt}, the original terms in Lagrangian reduce to 
\begin{IEEEeqnarray*}{rCl}
\IEEEeqnarraymulticol{3}{l}{\frac{\nu l}{2} \|w\|^2 - \nu l \rho  - \sum_k \alpha_k \left[(w\cdot\phi(x_k))  - \rho \right]} \\
&\qquad=& \frac{\nu l}{2} \|w\|^2  - \sum_k \alpha_k (w\cdot\phi(x_k)) \\
&\qquad=& - \frac{1}{2\nu l}\sum_{k, m}\alpha_k\alpha_m (\phi(x_k)\cdot\phi(x_m)).
\end{IEEEeqnarray*}
For the remaining privileged Lagrangian terms we use partial derivative expressions from  \eqref{delta_opt_limint}, 
then from \eqref{delta_opt}, and, finally, substitute the solution \eqref{w_star_opt} for $w^*$:
\begin{IEEEeqnarray*}{rClCl}
    \IEEEeqnarraymulticol{3}{l}{\frac{\gamma}{2} \|w^*\|^2}
    &+& \sum_k [(w^*\cdot\phi^*(x_k^*)) + b^*] (1-\alpha_k-\beta_k) \\
    &&&+&\sum_k \zeta_k (1-\beta_k -\mu_k) \\
    &\qquad=& \IEEEeqnarraymulticol{3}{l}{\frac{\gamma}{2} \|w^*\|^2 + \sum_k (w^*\cdot\phi^*(x_k^*))   (1-\alpha_k-\beta_k)} \\
    &\qquad=& \IEEEeqnarraymulticol{3}{l}{-\frac{1}{2\gamma}\sum_{k, m}(\alpha_k - \delta_k)(\alpha_m - \delta_m)(\phi^*(x_k^*)\cdot\phi^*(x_m^*)).}
\end{IEEEeqnarray*}
The sum of these original and privileged parts yields \eqref{dual_problem}.

In addition, lets write down Karush-Kuhn-Tucker conditions:
\begin{align}
  \alpha_k (f(x_k) + f^*(x_k^*)) & =0, \label{KKT_1} \\ 
  \beta_k (f^*(x_k^*) + \zeta_k) & = 0, \label{KKT_2} \\
  \mu_k \zeta_k & = 0. \label{KKT_3}
\end{align} 
From \eqref{KKT_1} we obtain \eqref{alpha_0} and \eqref{alpha_g0}.
When $\delta_k > 0$, then, according to \eqref{delta_opt_limint}, $\mu_k=\delta_k > 0$ and, thus, $\zeta_k=0$  \eqref{KKT_3}.
If $\delta_k < 1$, then $\beta_k >0$, and from \eqref{KKT_2} we have 
$f^*(x_k^*) + \zeta_k = 0$.
Thus, for $0<\delta_k<1$ we have \eqref{delta_01}.
If $\delta_k = 0$, then $\beta_k=1$, $\mu_k=0$, so from \eqref{KKT_3} it follows that $\zeta_k \geq 0$, and, according to \eqref{KKT_2}, 
we obtain $f^*(x_k^*) = -\zeta_k \leq 0$, i.e. \eqref{delta_0}.
If $\delta_k = 1$, then $\beta_k=0$, $\mu_k=1$, and so, from \eqref{KKT_2} and \eqref{KKT_3}, we have $\zeta_k=0$ and $f^*(x_k^*) \geq 0$, i.e. \eqref{delta_1}.
\section{Two-variable optimization problem} \label{appendix2}
The function \eqref{gsmo_objective} follows from applying the updates 
\eqref{gsmo_update} to the objective in the dual problem \eqref{dual_problem}. 
First, consider the original part.
Separating the terms with two chosen indices $i$ and $j$ and using the updates 
$\alpha_i \leftarrow \alpha_i + t$ and $\alpha_j \leftarrow \alpha_j - t$, we have
\begin{IEEEeqnarray*}{rCl}
  \IEEEeqnarraymulticol{3}{l}{\sum_{k, m}\alpha_k\alpha_m K(x_k, x_m) = (\alpha_i + t)^2 K(x_i, x_i)} \\
  &\>+& 2(\alpha_i + t)(\alpha_j - t) K(x_i, x_j) + (\alpha_j - t)^2 K(x_j, x_j) \\
  &\>+& 2(\alpha_i + t) \sum_{k\notin \{i, j\}}\alpha_k K(x_i, x_k) \\
  &\>+& 2 (\alpha_j - t)\sum_{k\notin \{i, j\}}\alpha_k K(x_j, x_k) \\
  &\>+& \sum_{k, m \notin \{i, j\}}\alpha_k\alpha_m K(x_k, x_m).
\end{IEEEeqnarray*}
After grouping the quadratic $t^2$ and linear $t$ terms it is reduced to  
\begin{multline*}
  t^2\left( K(x_i, x_i) - 2K(x_i, x_j) + K(x_j, x_j) \right)  \\
  +2t \left( \sum_{k}\alpha_k K(x_i, x_k) - \sum_{k}\alpha_k K(x_j, x_k) \right) + const.
\end{multline*}

Regarding privileged part, the updates become
$\alpha_i - \delta_i \leftarrow \alpha_i - \delta_i + t-s$, and 
$\alpha_j - \delta_j \leftarrow \alpha_j - \delta_j - (t-s)$,
and, thus, the quadratic form is similar up to substitution
$\alpha-\delta$, $t-s$, $K^*$ and $x^*$ instead of $\alpha$, $t$, $K$ and $x$ respectively.
\section{Proofs} \label{appendix3}
\begin{proof}[Proof of Lemma~\ref{lemma_1}]
To prove \ref{lemma_1_KK_pos}, let's assume $K_{ij}+K_{ij}^*=0$ in the case \ref{lemma_1_alpha} and $K_{ij}^*=0$ in the case \ref{lemma_1_delta}.
Since, 
$K_{ij} = \|\phi(x_i) - \phi(x_j)\|^2_{\ell_2}/\nu l$ and 
$K_{ij}^* = \|\phi^*(x_i^*) - \phi^*(x_j^*)\|^2_{\ell_2}/\gamma$, then
$\phi(x_i) = \phi(x_j)$, $\phi^*(x_i^*) = \phi^*(x_j^*)$ in the case \ref{lemma_1_alpha} and 
$\phi^*(x_i^*) = \phi^*(x_j^*)$ in the case \ref{lemma_1_delta}.
From \eqref{f_k} and \eqref{f_k_star} we have
\begin{equation*}
\begin{IEEEeqnarraybox}{rCl} 
    \Delta f_{ij}   &=& \frac{1}{\nu l} \sum\limits_{k=1}^l\alpha_k \langle \phi(x_k) \cdot (\phi(x_i)-\phi(x_j))\rangle, \\
    \Delta f_{ij}^* &=& \frac{1}{\gamma}\sum\limits_{k=1}^l(\alpha_k - \delta_k) \langle\phi^*(x_k^*) \cdot (\phi^*(x_i^*)-\phi^*(x_j^*))\rangle. 
\end{IEEEeqnarraybox}
\end{equation*}
Therefore, in our two cases we obtain, respectively, $\Delta f_{ij} = \Delta f_{ij}^* = 0$ and $\Delta f_{ij}^* = 0$,
which is wrong, due to violations $\Delta f_{ij} + \Delta f_{ij}^* > 0$ and $\Delta f_{ij}^* > 0$. 
The part \ref{lemma_1_KK_pos} is proved.

To prove \ref{lemma_ts_opt} and \ref{lemma_1_ts_sign} let's first consider the case \ref{lemma_1_alpha}.
Since the $\alpha$-pair is violating, 
then $\Delta f_{ij} + \Delta f_{ij}^* > 0$.
Then the minimum of \eqref{gsmo_objective_t} is either the extremum 
$t=-(\Delta f_{ij} + \Delta f_{ij}^*)/(K_{ij}+K_{ij}^*)<0$ or the bound point $t=-\alpha_i<0$.
The case \ref{lemma_1_delta} is similar. 
We have $\Delta f_{ij}^* > 0$.
Then the minimum is either the extremum 
$s=\Delta f_{ij}^*/K_{ij}^*>0$ or the bound $s=\min{(1-\delta_i, \delta_j)}>0$.
The parts \ref{lemma_ts_opt} and \ref{lemma_1_ts_sign} are proved.

In the case \ref{lemma_1_alpha}, after a substitution of $t=-(\Delta f_{ij} + \Delta f_{ij}^*)/(K_{ij}+K_{ij}^*)$, $s=0$ into \eqref{delta_f_update}, we have $\Delta f_{ij}+\Delta f_{ij}^*\leftarrow 0$,
and, thus, \eqref{KKT_a} is satisfied.
Similarly, in the case \ref{lemma_1_delta}, substituting $t=0$, $s=\Delta f_{ij}^*/K_{ij}^*$ into \eqref{delta_f_update} we have $\Delta f_{ij}^*\leftarrow 0$, and, thus, \eqref{KKT_d} is satisfied.
The part \ref{lemma_1_Deltas} is proved.
For on-bound points we have $\alpha_i \leftarrow 0$ in the case \ref{lemma_1_alpha} and $\delta_i \leftarrow 1$ or $\delta_j \leftarrow 0$ in the case \ref{lemma_1_delta}. 
So, in both cases the pairs become not a $\tau$-violating.
The part \ref{lemma_1_not_violating} is proved.
\end{proof}
\begin{proof}[Proof of Lemma~\ref{lemma_2}]
Every iteration in the algorithm updates 
a $\tau$-violating pair $\alpha_i^\mathrm{new}=\alpha_i+t_\mathrm{opt}$, $\alpha_j^\mathrm{new}=\alpha_j-t_\mathrm{opt}$ or $\delta_i^\mathrm{new}=\delta_i+s_\mathrm{opt}$, $\delta_j^\mathrm{new}=\delta_j-s_\mathrm{opt}$, 
where $t_\mathrm{opt}$ and $s_\mathrm{opt}$ are the optimal solutions of \eqref{gsmo_objective_t} and \eqref{gsmo_objective_s}, respectively.
In the former case we use the fact that 
if $a>0$, $b>0$, then $-at^2-2bt \geq -bt$ for $-b/a \leq t \leq 0$.
It means that this parabola lies above the line joining the origin and the maximum point $(-b/a, b^2/a)$.
So, given that 
$\Delta f_{ij} + \Delta f_{ij}^* > \tau$ and  
$||\bm\alpha_\mathrm{new} - \bm\alpha||=
\sqrt{(\alpha_i^\mathrm{new}-\alpha_i)^2+(\alpha_j^\mathrm{new}-\alpha_j)^2}=
-t_\mathrm{opt}\sqrt{2}$, we obtain
\begin{multline*}
    F(\bm\alpha) - F(\bm\alpha_\mathrm{new}) = -\varphi(t_\mathrm{opt}) \geq -(\Delta f_{ij} + \Delta f_{ij}^*)t_\mathrm{opt} \\ 
    > \tau(-t_\mathrm{opt}) = \frac{\tau}{\sqrt{2}}||\bm\alpha_\mathrm{new}-\bm\alpha||.
\end{multline*}

Similarly, in the case 
$\delta_i^\mathrm{new} = \delta_i+s_\mathrm{opt}$, 
$\delta_j^\mathrm{new} = \delta_j-s_\mathrm{opt}$
we use that 
if $a>0$, $b>0$, then $-as^2+2bs \geq bs$ for $0 \leq s \leq b/a$.
By $\Delta f_{ij}^* > \tau$ and 
$||\bm\alpha_\mathrm{new} - \bm\alpha||=
\sqrt{(\delta_i^\mathrm{new}-\delta_i)^2+(\delta_j^\mathrm{new}-\delta_j)^2}=
s_\mathrm{opt}\sqrt{2}$,
\begin{multline*}
    F(\bm\alpha) - F(\bm\alpha_\mathrm{new}) = -\varphi^*(s_\mathrm{opt}) \geq \Delta f_{ij}^*\>s_\mathrm{opt} > \tau s_\mathrm{opt} \\ = \frac{\tau}{\sqrt{2}}||\bm\alpha_\mathrm{new}-\bm\alpha||.
\end{multline*}
\end{proof}
\begin{proof}[Proof of Proposition~\ref{prop_inf_convergence}]
The constraints in \eqref{dual_problem} provide a feasible set $\mathcal{F}$
which is compact,
because for $\bm\alpha\in \mathcal{F}$ the inequalities  $0\leq\alpha_k\leq \nu l$, $0\leq\delta_k\leq \nu l$ are nonstrict.
So, the objective $F(\bm\alpha)$ is bounded below on $\mathcal{F}$.  
Thus, since $\{F(\bm\alpha(i))\}$ is a decreasing sequence, there exists
$\overline{F}$ such that $\lim_{i\to\infty} F(\bm\alpha(i))=\overline{F}$.
By Lemma~\ref{lemma_2} we have
\begin{IEEEeqnarray*}[\IEEEeqnarraystrutmode\IEEEeqnarraystrutsizeadd{\IEEEspace}{\IEEEspace}]{r}
    F(\bm\alpha(i)) - F(\bm\alpha(i+1)) > \frac{\tau}{\sqrt{2}}||\bm\alpha(i) - \bm\alpha(i+1)||, \qquad \\ 
    i=0, 1, \ldots,\\
    \IEEEeqnarraymulticol{1}{c}{\vdots} \\
    \IEEEeqnarraymulticol{1}{l}{F(\bm\alpha(i+l-1)) - F(\bm\alpha(i+l))} \\ 
    \qquad > \frac{\tau}{\sqrt{2}}||\bm\alpha(i+l-1) - \bm\alpha(i+l)||, \quad i,l=0, 1, \ldots.  
\end{IEEEeqnarray*}
Summing up these inequalities and repeatedly applying triangle inequality,
we have
\begin{multline*}
    F(\bm\alpha(i)) - F(\bm\alpha(i+l)) > \frac{\tau}{\sqrt{2}}||\bm\alpha(i) - \bm\alpha(i+l)||, \\ i, l=0, 1, \ldots,
\end{multline*}
meaning that $\{\bm\alpha(i)\}$ is a Cauchy sequence.
Both the initialization and the iteration of the algorithm ensures that 
$\bm\alpha(i)\in \mathcal{F}$.
So, since $\mathcal{F}$ is compact, then $\{\bm\alpha(i)\}$ converges to a vector in
$\mathcal{F}$.
\end{proof}
\begin{proof}[Proof of part \ref{lemma_fin_convergence_to_int} of  Lemma~\ref{lemma_fin_convergence}]
Let's assume that $\alpha$-pair can move to the point
$(\alpha_p^\mathrm{new}, \alpha_q^\mathrm{new})$ such that 
$\alpha_p^\mathrm{new} > 0$ infinitely number of times.
Each time it happens, by \ref{lemma_1_Deltas} in Lemma~\ref{lemma_1} there is 
$f_p+f_p^*=f_q+f_q^*$.
Moreover, 
by the continuity of functions $f_i$, $f_i^*$, $i=1,\ldots,l$, with respect to ${\bm\alpha}$, according to Proposition~\ref{prop_inf_convergence}, there are limits $\overline{f}_i=\lim\limits_{k\to\infty}f_i$, $\overline{f}_i^*=\lim\limits_{k\to\infty}f_i^*$.
Therefore, $\forall~\epsilon>0$ there is an iteration such that at the all subsequent ones we will have 
$|f_p+f_p^*-(\overline{f}_p+\overline{f}_p^*)| < \epsilon$ and
$|f_q+f_q^*-(\overline{f}_q+\overline{f}_q^*)| < \epsilon$.
This leads to $\overline{f}_p+\overline{f}_p^* = \overline{f}_q+\overline{f}_q^*$
and $f_p+f_p^*-(f_q+f_q^*)<2\epsilon<\tau$, 
meaning that $(p, q)$ remains not a $\tau$-violating, which contradicts to 
$(p, q)\in I_{\alpha}^{\infty}$.

Similarly, for $\delta$-pair,
we have $\overline{f}_p^* = \overline{f}_q^*$ and
$f_p^*-f_q^*<2\epsilon<\tau$, which contradicts to 
$(p, q)\in I_{\delta}^{\infty}$.
\end{proof}
\begin{proof}[Proof of part \ref{lemma_fin_convergence_from_int} of  Lemma~\ref{lemma_fin_convergence}]
Let's first consider $\alpha$-pairs.
Let's assume that $\alpha_q>0$ can be infinitely many times.
According to \ref{lemma_fin_convergence_to_int},
the step is $(\alpha_p, \alpha_q)\to(0, \alpha_q^{\textrm{new}})$.
Each such step decreases the number of positive coefficients $\alpha_i$ by one.
Since the number of data examples $l$ is finite, 
then to increase the number of non-zero coefficients there must be  
infinitely many steps like $(\alpha_i, 0)\to(\alpha_i^{\textrm{new}}, \alpha_j^{\textrm{new}})$, $\alpha_i^{\textrm{new}}>0$.
For this there must be at least one pair $(r, s)$ such that $\alpha_r^{\textrm{new}}>0$ infinitely many times, which is impossible due to \ref{lemma_fin_convergence_to_int}.

Similarly, for $\delta$-pairs we assume that there can be infinitely many steps
$(\delta_p, \delta_q)\to(1, \delta_q^{\textrm{new}})$, $\delta_p>0$, and
$(\delta_p, \delta_q)\to(\delta_p^{\textrm{new}}, 0)$, $\delta_q<1$.
Each of them decreases the number of at-bound coefficients by one or two.
In the former there must be at least one pair $(r, s)$ such that 
$(\delta_m, 1)\to(\delta_m^{\textrm{new}}, \delta_n^{\textrm{new}})$, $\delta_n^{\textrm{new}}>0$, whereas in the latter
$(0, \delta_n) \to (\delta_m^{\textrm{new}}, \delta_n^{\textrm{new}})$,
$\delta_m^{\textrm{new}}<1$, infinitely many times, which is impossible due to \ref{lemma_fin_convergence_to_int}.
\end{proof}
\begin{proof}[Proof of part \ref{lemma_fin_convergence_limits} of 
Lemma~\ref{lemma_fin_convergence}]
Let's consider the following placements of the limit point
$(\overline{\alpha}_p, \overline{\alpha}_q)$:
1) $\mathrm{int}S_{\alpha}\cup{S_{\alpha}^{ne}}$;
2) $S_{\alpha}^{s}\cup{(\nu l, 0)}$;
3) $S_{\alpha}^{w}\cup{(0, \nu l)}$;
4) $(0, 0)$.
Similarly, for $(\overline{\delta}_p, \overline{\delta}_q)$ we consider:
1) $\mathrm{int}S_{\delta}$;
2) $S_{\delta}^{w}\cup{S_{\delta}^{n}}\cup{(0, 1)}$;
3) $S_{\delta}^{s}\cup{S_{\delta}^{e}}\cup{(1, 0)}$;
4) $(0, 0)\cup{(1, 1)}$.
Next, we show, that the variants 1)-3) are impossible in both cases.
Indeed, the cases 1) are impossible due to \ref{lemma_fin_convergence_to_int} and, at the same time, due to \ref{lemma_fin_convergence_from_int} (the steps move from and to not on-bound points).
2) are impossible by \ref{lemma_fin_convergence_to_int} (movements to not on-bound points); 
3) are impossible by \ref{lemma_fin_convergence_from_int} (movements from not on-bound points).
The only variants 4) are possible, and,
according to \ref{lemma_fin_convergence_to_int} and \ref{lemma_fin_convergence_from_int}, the steps can be only between bounds: $S_{\alpha}^s\to S_{\alpha}^w$ for $\alpha$-steps and $S_{\delta}^w\to S_{\delta}^s$ or $S_{\delta}^n\to S_{\delta}^e$ for $\delta$-steps.  
\end{proof}
\begin{proof}[Proof of part \ref{lemma_fin_convergence_qp} of  Lemma~\ref{lemma_fin_convergence2}]
First, let's consider $\alpha$-pairs.
Assume that both 
$(p, q)\in I_{\alpha}^{\infty}$ and 
$(q, p)\in I_{\alpha}^{\infty}$.
Then at the iterations from $L_{\alpha}(p, q)$ 
the pair $(p, q)$ is $\tau$-violating,
i.e. $f_p+f_p^*-(f_q+f_q^*)>\tau$ by \eqref{Deltas}.
Also, at the iterations from
$L_{\alpha}(q, p)$ we have $f_q+f_q^*-(f_p+f_p^*)>\tau$.
Since $f_i$ and $f_i^*$, $i=1,\ldots,l$, are continuous with respect to all coefficients in ${\bm\alpha}$, 
then by Proposition~\ref{prop_inf_convergence}  there are the limits 
$\overline{f}_i=\lim\limits_{k\to\infty}f_i$ and
$\overline{f}_i^*=\lim\limits_{k\to\infty}f_i^*$.
Thus, we have a contradiction:
$\overline{f}_p+\overline{f}_p^*-(\overline{f}_q+\overline{f}_q^*) \geq \tau$
and 
$\overline{f}_q+\overline{f}_q^*-(\overline{f}_p+\overline{f}_p^*) \geq \tau$.

Similarly, for $\delta$-pairs we assume that both 
$(p, q)\in I_{\delta}^{\infty}$ and 
$(q, p)\in I_{\delta}^{\infty}$,
which leads to a contradiction:
$\overline{f}_p^*-\overline{f}_q^* \geq \tau$
and 
$\overline{f}_q^*-\overline{f}_p^* \geq \tau$.
\end{proof}
\begin{proof}[Proof of part \ref{lemma_fin_convergence_rs} of 
Lemma~\ref{lemma_fin_convergence2}]
According to Lemma~\ref{lemma_fin_convergence}, after a sufficiently large number of iterations, $\alpha$-steps can only have the form $(\alpha_p, 0)\to(0, \alpha_q^{\mathrm{new}})$, 
whereas $\delta$-steps are like $(0, \delta_q)\to(\delta_p^{\mathrm{new}}, 0)$ or 
$(\delta_p, 1)\to(1, \delta_q^{\mathrm{new}})$.
For the $\alpha$-pair $(p, q)$ to move infinitely many times,
it must return back $S_{\alpha}^w\to S_{\alpha}^s$ ($S_{\delta}^s \to S_{\delta}^w$ or $S_{\delta}^e \to S_{\delta}^n$).
Direct return is impossible due to \ref{lemma_fin_convergence_qp}.
Therefore, there must be an example $r \neq p$ such that
the pair $(r, p)$, is $\tau$-violating infinitely many times and
the step $(\alpha_r, 0) \to (0, \alpha_p^\mathrm{new})$, $\alpha_r>0$, $\alpha_p=0$, 
makes $\alpha_p>0$ again.
Moreover, there must be an example $s \neq q$ to move the pair $(q, s)$, 
by $(\alpha_q, 0)\to (0, \alpha_s^\mathrm{new})$, 
$\alpha_q>0$, $\alpha_s=0$, making $\alpha_q=0$ again.

Similarly, 
for $\delta$-pairs in the case of steps near the limit $(0, 0)$ we have
$(0, \delta_p) \to (\delta_r^\mathrm{new}, 0)$, $\delta_r=0$, $\delta_p>0$, and 
$(0, \delta_s) \to (\delta_q^\mathrm{new}, 0)$, $\delta_q=0$, $\delta_s>0$.
In the case of steps near $(1, 1)$ we have
$(\delta_r, 1) \to (1, \delta_p^\mathrm{new})$, $\delta_r<1$, $\delta_p=1$, and
$(\delta_q, 1) \to (1, \delta_s^\mathrm{new})$, $\delta_q<1$, $\delta_s=1$.
\end{proof} 
\begin{proof}[Proof of Theorem~\ref{theorem_fin_convergence}]
Assume that our algorithm does not stop. 
Then there exists a $\tau$-violating pair 
$(p, q)\in I_{\alpha}^\infty$ or $(p, q)\in I_{\delta}^\infty$.
For $\alpha$-pairs by \eqref{Deltas} we have $f_p+f_p^*-(f_q+f_q^*)>\tau$ and for $\delta$-pairs we have $f_p^*-f_q^*>\tau$.
Since $f_i$ and $f_i^*$, $i=1,\ldots,l$, are continuous with respect to all coefficients in ${\bm\alpha}$, 
then by Proposition~\ref{prop_inf_convergence} there are the limits 
$\overline{f}_i=\lim\limits_{k\to\infty}f_i$ and
$\overline{f}_i^*=\lim\limits_{k\to\infty}f_i^*$, and, therefore,
$\overline{f}_p+\overline{f}_p^*-(\overline{f}_q+\overline{f}_q^*) \geq \tau>0$.
According to \ref{lemma_fin_convergence_rs} in Lemma~\ref{lemma_fin_convergence2},
there are pairs $(r, p)$ and $(q, s)$ selected infinitely many times, and, therefore
\begin{equation*}
    \overline{f}_r+\overline{f}_r^* > \overline{f}_p+\overline{f}_p^* > \overline{f}_q+\overline{f}_q^* > \overline{f}_s+\overline{f}_s^*.
\end{equation*}
By these inequalities all examples $r, p, q, s$ are different.
After applying the same logic to the pairs $(r, p)$ and $(q, s)$ we conclude that we need additional examples $r'$ and $s'$ to form pairs $(r', r)$ and $(s, s')$ such that
\begin{multline*}
    \overline{f}_{r'}+\overline{f}_{r'}^* >
    \overline{f}_r+\overline{f}_r^* > \overline{f}_p+\overline{f}_p^* > \\ \overline{f}_q+\overline{f}_q^* > \overline{f}_s+\overline{f}_s^*
    > \overline{f}_{s'}+\overline{f}_{s'}^*.
\end{multline*}
From this it follows that all data examples are different.
Since this counting requires infinite repetition, we need an infinite number of data examples to satisfy $(p, q)\in I_{\alpha}^\infty$, which is impossible.

Similarly, for $(p, q)\in I_{\delta}^\infty$ we have
\begin{equation*}
    \overline{f}_{r'}^* > \overline{f}_r^* > \overline{f}_p^* > \\ 
    \overline{f}_q^* > \overline{f}_s^* > \overline{f}_{s'}^*,
\end{equation*}
and we need infinitely many data examples.
Since of this contradiction, the algorithm stops after a finite number of iterations.
\end{proof}
\end{appendices}
\bibliographystyle{IEEEtran}
\bibliography{literature}
\begin{IEEEbiography}[{\includegraphics[width=1in,height=1.25in,clip,keepaspectratio]{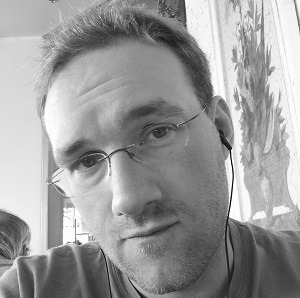}}]{Andrey Lange}
received the M.Sc. degree in applied mathematics and the Ph.D. degree in probability theory and statistics from Bauman Moscow State Technical University (BMSTU), in 2002 and 2007, respectively.
He is currently a Senior Research Engineer with the Skolkovo Institute of Science and Technology (Skoltech) and a Research Fellow with the Federal Research Center ``Computer Science and Control'' of Russian Academy of Sciences (FRC CSC RAS).
His research interests include stochastic processes and machine learning and its applications.
\end{IEEEbiography}
\begin{IEEEbiography}[{\includegraphics[width=1in,height=1.25in,clip,keepaspectratio]{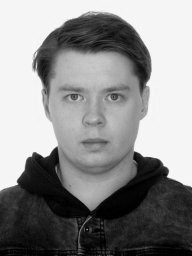}}]{Dmitry Smolyakov}
received the degree from the Moscow Institute of Physics and Technology and the Ph.D. degree from the Skolkovo Institute of Science and Technology (Skoltech).
He is currently a Junior Research with the institute for Information Transmission Problems of Russian Academy of Sciences (ITTP RAS) and a Data Scientist in the Industry.
\end{IEEEbiography}
\begin{IEEEbiography}[{\includegraphics[width=1in,height=1.25in,clip,keepaspectratio]{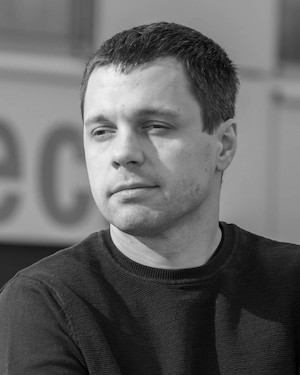}}]{Evgeny Burnaev}
received the M.Sc. degree in applied physics and mathematics from the Moscow Institute of Physics and Technology, in 2006, and the Ph.D. degree in foundations of computer science from the Institute for Information Transmission Problems, Russian Academy of Sciences (RAS), in 2008.
He is currently a Full Professor with the Skolkovo Institute of Science and Technology and the Director of the Skoltech Applied AI Center.
His research interests include regression based on Gaussian processes and kernel methods for multi-fidelity surrogate modeling and optimization, deep learning for 3D data analysis and manifold learning, online sequence learning for prediction, and non-parametric anomaly detection.
\end{IEEEbiography}
\EOD
\end{document}